\documentclass{tlp}

\usepackage{amssymb}
\usepackage{hyperref}

\usepackage{amsmath}
\usepackage{graphicx}
\usepackage{multirow}

\usepackage{placeins}

\let\proof\relax

\usepackage{amsthm}

\newtheorem{theorem}{Theorem}[section]
\newtheorem*{theorem*}{Theorem}
\newtheorem{corollary}[theorem]{Corollary}
\newtheorem{proposition}[theorem]{Proposition}

\theoremstyle{definition}

\newtheorem{definition}[theorem]{Definition}

\usepackage{subcaption}

\usepackage{algorithm}
\usepackage[noend]{algpseudocode}
\algnewcommand\algorithmicforeach{\textbf{for each}}
\algdef{S}[FOR]{ForEach}[1]{\algorithmicforeach\ #1\ \algorithmicdo}

\usepackage{mathtools}
\usepackage{bm}
\usepackage{esvect}

\newcommand{\point}{{\ensuremath{P}}}
	
\newcommand{\PointSet}{{\ensuremath{\cal{P}}}}
\newcommand{\distSym}{\ensuremath{d}}
\newcommand{\dist}[2]{{\ensuremath{\distSym(#1,#2)}}}
\newcommand{\segment}[2]{{\ensuremath{\overline{#1 #2}}}}
\newcommand{\StraightLine}[2]{{\ensuremath{\overleftrightarrow{#1 #2}}}} \newcommand{\Path}{{\ensuremath{p}}} \newcommand{\PathFromTo}[2]{{\ensuremath{{\Path}_{#1-#2}}}} \newcommand{\Circuit}{{\ensuremath{c}}}
\newcommand{\Length}[1]{{\ensuremath{L(#1)}}}

\newcommand{\seq}{{\ensuremath{s}}} \newcommand{\weight}[1]{{\ensuremath{w(#1)}}}
\newcommand{\Npoints}{{\ensuremath{n}}}	

\newcommand{\ClusterSet}[1]{{\ensuremath{C (#1)}}}
\newcommand{\ClusterOp}[1]{{\ensuremath{C (#1)}}}
\newcommand{\Vicini}[1]{{\ensuremath{{\cal N} (#1)}}}

\newcommand{\Next}{{\ensuremath{\mathit{Next}}}}	
\newcommand{\Prev}{{\ensuremath{\mathit{Prev}}}}	

\newcommand{\nocrossing}{{\tt nocrossing}}
\newcommand{\Dom}[1]{{\ensuremath{Dom(#1)}}}

\newcommand{\Angle}[1]{{\ensuremath{\angle #1}}}
\newcommand{\ConvexHull}[1]{{\ensuremath{{\cal H}(#1)}}}
\newcommand{\Hull}{{\ensuremath{H}}}	

\newcommand{\HullSet}[1]{{\ensuremath{\partial{\cal H}(#1)}}}	
\newcommand{\HullSetP}{{\ensuremath{\partial{\cal H}(\PointSet)}}}
\newcommand{\HullSetI}{{\ensuremath{\partial{\cal H}(I)}}}	
\newcommand{\NHull}{{\ensuremath{\vert \HullSetP \vert}}} 
\newcommand{\HSucc}[2]{\ensuremath{S(#1)}}

\newcommand{\circuit}{{\tt circuit}}
\newcommand{\alldifferent}{{\tt alldifferent}}

\newcommand{\minAlpha}{\ensuremath{\underline{\alpha}}}
\newcommand{\maxBeta}{\ensuremath{\overline{\beta}}}

\newcommand{\ECLiPSe}{ECL$^i$PS$^e$}

\newcommand{\noi}{{GEO}}
\newcommand{\loro}{{HK-RMC}}
\newcommand{\noipiuloro}{{GEO+HK-RMC}}

\newcommand{\reference}{ECLP}

\newcommand{\maxregret}{{\tt MAX-REGRET}}
\newcommand{\lcfirst}{{\tt LC\_FIRST MAX\_COST}}

\newcommand{\NP}{\textnormal{NP}}

\newcommand{\NPH}{\textsc{\NP-hard}}

\newcommand{\Constraints}{\mathcal{C}}

\newcommand{\Domains}{\mathcal{D}}
\newcommand{\Variables}{\mathcal{X}}
\usepackage{mathrsfs} 
\newcommand{\CSProblem}{{\ensuremath{\mathscr{P}}}}

\usepackage{acronym}
\newacro{TSP}{Traveling Salesperson Problem}
\newacro{TSPTW}{TSP with Time Windows}
\newacro{CP}{Constraint Programming}
\newacro{CSP}{Constraint Satisfaction Problem}
\newacro{COP}{Constraint Optimization Problem}
\newacro{PTAS}{Polynomial Time Approximation Scheme}
\newacro{VRP}{Vehicle Routing Problem}
\newacro{CLP}{Constraint Logic Programming}
\newacro{LHS}{Latin Hypercube Sampling}
\newacro{RF}{Random Forest}
\newacro{MLP}{Multi-Layer Perceptron}
\newacro{NN}{Neural Network}
\newacro{CLP(FD)}{CLP on Finite Domains}
\newacro{LKH}{Lin-Kernighan-Helsgaun}
\newacro{WCC}{Weighted Circuit Constraint}
\newacro{ILP}{Integer Linear Programming}
\newacro{SCC}{Strongly Connected Components}
\newacro{MST}{Minimum Spanning Tree}
\newacro{ETSP}{Euclidean Traveling Salesperson Problem}
\newacro{CNF}{Conjunctive Normal Form}
\newacro{DNF}{Disjunctive Normal Form}
\newacro{PUR}{Positive Unit Resolution}
\newacro{AI}{Artificial Intelligence}
\newacro{LP}{Logic Programming}
\newacro{GAC}{Generalized Arc Consistency}
\newacro{GTSP}{Generalized Traveling Salesperson Problem}
\newacro{EGTSP}{Euclidean Generalized Traveling Salesperson Problem}

\begin{document}

\lefttitle{Bertagnon A. and Gavanelli M.}

\jnlPage{1}{8}
\jnlDoiYr{2021}
\doival{10.1017/xxxxx}

\title[Enhanced Filtering Algorithms for the ETSP and its variants in CLP]
{Enhanced Filtering Algorithms for the Euclidean Traveling Salesperson Problem and its variants in Constraint Logic Programming}

\begin{authgrp}
\author{ALESSANDRO BERTAGNON}
\affiliation{Department of Environmental and Prevention Sciences, University of Ferrara, Italy}
\author{MARCO GAVANELLI}
\affiliation{Department of Engineering, University of Ferrara, Italy}
\end{authgrp}


\maketitle

\begin{abstract}
The \ac{TSP} is one of the best-known problems in computer science
and arises in many engineering applications, such as smart vehicles and intelligent transportation
systems. In the ``Euclidean'' case, each node is defined by its coordinates in the plane and distances are computed using the Euclidean metric. In the Constraint Programming (CP) literature, the Euclidean TSP is typically addressed by computing the full distance matrix and treating it as a general case; however this approach ignores the geometric information carried by the points' coordinates. In this work, we propose new filtering algorithms, implemented in \ac{CLP}, that exploit such geometric information to achieve stronger constraint propagation than existing approaches. Moreover, we show how this methodology can be extended to other Euclidean variants of the \ac{TSP}, including the Euclidean Generalized Traveling Salesperson Problem (EGTSP), which is relevant in practical routing and logistics applications. Experimental results demonstrate the computational advantages of the proposed approach.
\end{abstract}

\begin{keywords}
Euclidean Traveling Salesperson Problem, Constraint Logic Programming, Geometric Filtering Algorithms, Euclidean Generalized Traveling Salesperson Problem
\end{keywords}


\acresetall 

\section{Introduction}
\sloppy

Many engineering applications, including intelligent transportation systems, logistics, smart vehicles, and pickup-and-delivery services, require solving complex routing problems. These problems involve a vehicle or fleet of vehicles reaching a set of destinations while minimizing travel distance to save time, money, energy, and, in some cases, reduce pollution.

Although each application has its own peculiarities, the problem formulation shares an omnipresent core, named in the literature the {\em \ac{TSP}}; the name comes from the problem's most famous formulation: \textit{``A salesman has to visit a set of cities, each of which must be visited only once, and he wants to minimize the length of the tour"}.
More formally, given a weighted graph, the \ac{TSP} requires to compute the minimum cost cycle that visits each vertex exactly once.

Most routing problems are intractable, typically NP-hard.
While heuristic approaches can quickly provide good solutions in some cases, systematic algorithms are preferred in others. These algorithms can find the optimal solution, complete with a proof of optimality (in the eXplainable AI spirit) given enough time.
Clearly, obtaining the provably optimal solution adds significant value.

The basic TSP problem attracted the attention of researchers in different research areas, including operations research and artificial intelligence,
with solution methods ranging from Integer Linear Programming, to Constraint Programming, to metaheuristics, just to name a few
\citep{CaseauKoppstein93, DBLP:conf/cpaior/MelgarejoLS15, DBLP:conf/aaai/DucommanCP16, Boudreault2021}.
Beside the basic TSP problem, engineering applications involve a series of other requirements; the need to take into account such
requirements spawned a plethora of
extensions to the TSP, e.g., requiring each node to be visited in a time window (TSP with time windows), assuming that more than one agent (or vehicle) is available (Multiple TSP),
or that not all nodes must be visited.
For example, in the \ac{GTSP} \citep{LAPORTE1987,LAPORTE1981} (also known as \emph{set \ac{TSP}}) 
the nodes of the graph are partitioned into various sets, called \emph{clusters}; an optimal solution of the \ac{GTSP} is a minimum cost cycle that visits at least one node in each cluster.
The \ac{GTSP} models situations in which it is necessary to collect resources in an environment where several equivalent providers are available for each resource type, while minimizing travel time, or energy requirements from a battery.

In the theoretical formulation of the TSP, the weights on the arcs can take any positive value without restrictions;
however in most practical cases it is reasonable to assume that the distance function obeys the
triangular inequality, where going directly from a node $A$ to a node $B$ cannot be more expensive than taking a detour to a third node $C$: 
such case is named {\em metric TSP} in the literature.
Even more, in many industrial applications,
the nodes are associated to points in the plane, they are connected through straight lines, 
and the distance is the length of the segment connecting the two nodes; this case is named the {\em \ac{ETSP}}, and many instances in the TSPLIB~\citep{tsplib} set of benchmarks fall in this category.
The techniques proposed in this paper are intended for complete Euclidean instances. They do not cover routing settings in which crossings are unavoidable, or in which replacing crossing connections with uncrossed ones would produce infeasible links, as may occur in real road networks. Such settings are outside the scope of the
present work.

From the computational complexity point of view, both the metric and the Euclidean TSP are \NPH~\citep{TSPeuclideoNPhard},
even though,
differently from the general TSP, the Euclidean TSP admits a \ac{PTAS} \citep{Arora_PTAS,Mitchell_PTAS}.

Currently, the best solver for the \ac{TSP} is the mathematical programming solver Concorde \citep{concorde}.  Concorde is based on \ac{ILP} techniques such as branch-and-cut
and it was used to optimally solve instances with tens of thousands of nodes.
However, most of the algorithms developed for solving the \ac{TSP}, including Concorde, can only be used when addressing the basic TSP, whereas when there are other additional constraints, it becomes necessary to adopt more general solving techniques. 

A general technique developed in Artificial Intelligence to solve \acp{COP} is
constraint propagation \citep{AC3}.
The idea was so successful that new languages were embedding it as part of their operational semantics;
for example the \ac{CLP} \citep{JaffarMaher94} class of languages
was defined to declaratively address constraint problems,
and many languages of this class exploited constraint propagation to efficiently solve constraint problems.
The \ac{CLP} research area was later extended to include also imperative and object-oriented programming languages, generating the \ac{CP} research area.

In \ac{CP}, programmers can easily define constraint optimization problems, including side constraints, without being limited to linear or convex constraints or objective functions. This flexibility makes CP ideal for solving industrial optimization problems. Over the years, numerous solvers have been developed, and today, a variety of open-source and commercial solvers are available and widely used in both academic and commercial applications.

Many works in the \ac{CP} literature address the \ac{TSP} \citep{CaseauLaburthe,KayaHooker,vanHoeveEtAl_ImprovedWeightedCircuit,DBLP:journals/corr/abs-1206-3437,SalesmanAndTree,DBLP:conf/cpaior/DeudonCLAR18}; the usual way to tackle Euclidean \acp{TSP} is to compute the distance matrix and address the problem as a general \ac{TSP}.
On the other hand, this approach disregards all the additional information available in the \ac{ETSP}: the coordinates of the points to be visited are known, and geometrical concepts (straight line segments, angles, etc.) can be defined in the Euclidean plane.
An interesting research question is: {\em ``Is it possible to exploit the additional information intrinsically available in Euclidean instances in order to speed up the search?''}
The initial answer was presented in a conference publication \cite{BertagnonGavanelliAAAI20}, where a significant speedup was obtained
in solving Euclidean instances in CLP, by exploiting efficient constraint propagation algorithms tailored for Euclidean TSP instances.
This article significantly extends that work~\citep{BertagnonGavanelliAAAI20},
providing deeper insights, comprehensive analysis, and additional findings.

As a first contribution, we provide a more precise description of the introduced algorithms and we extend the experimental campaign by testing the proposed algorithms on a large dataset of \ac{TSP} instances, including publicly available  and
synthetic instances.

Now, a further research question could be whether such enhanced pruning can also be applied beyond the TSP:
as a second contribution, this 
article shows that the additional pruning based on geometric properties is not confined only to the case of the Euclidean TSP, but can also be applied to other variants of routing problems.
For instance, we show how the proposed techniques can be applied to the \ac{EGTSP}.

The underlying no-crossing and convex-hull ordering properties are
classical geometric results; our contribution lies in translating
them into dedicated \ac{CP}/\ac{CLP} constraints and propagators, and
in extending this approach to the \ac{EGTSP}.

We believe that this work could inspire researchers to extend the devised algorithms also to other routing problems, and also to 
inspire new propagation algorithms to exploit the additional information present in geometric instances deeply.

The rest of the article is organized as follows. In Section~\ref{sec:preliminaries} we introduce the necessary definitions and notations. In Section~\ref{sec:related} we discuss related works. In Section~\ref{sec:nocrossing} we provide a description of our techniques that can exploit the geometric information present in the \ac{ETSP} to improve propagation. In Section~\ref{sec:egtsp_hull} we show how geometric information can also be exploited in other variants of routing problems; we take \ac{EGTSP} as an example. We show the computational results in Section~\ref{sec:experiments} and, finally, in Section~\ref{sec:conclusion} we conclude.
 \section{Preliminaries and notation} 
\label{sec:preliminaries}

A \acf{CSP} is a triple $\CSProblem = \langle \Variables,\Domains,\Constraints \rangle$ where $\Variables$ is a set of $n$ decision variables $\lbrace x_{1}, x_{2}, \dots , x_{n} \rbrace$, $\Domains$ is a set of domains $\lbrace \Dom{x_1}, \Dom{x_2},\dots, \Dom{x_n}\rbrace$ and $\Constraints$ is a set of constraints $\lbrace c_{1}, c_{2}, \dots, c_{m} \rbrace$. Each domain $\Dom{x_i}$ is the set of all possible values that can be assigned to the variable $x_{i}$. Each constraint $c_{i}$ consists of a pair $\langle R_{i}, S_{i} \rangle$ where $R_{i}$ is a relation between the variables $S_{i}$ participating in the constraint.

A solution for a CSP $\CSProblem$ is an {\em instantiation} that assigns a value to each variable in $\Variables$ and satisfies all the constraints $\Constraints$. 
While in a \ac{CSP} all solutions are equally good, in many practical cases an objective function discriminates between solutions.

A \acf{COP} is a pair $\langle \CSProblem,f \rangle$ where $\CSProblem = \langle \Variables,\Domains,\Constraints \rangle$ is a \ac{CSP} and $f$ a function $f : \Dom{x_{1}} \times \dots \times  \Dom{x_{n}} \rightarrow \mathbb{R}$ that associates a value to every solution of \CSProblem. A solution $s$ of $\CSProblem$ is an {\em optimal solution} of the COP $\langle \CSProblem,f \rangle$ 
if and only if there is no solution $s'$ of \CSProblem\ such that $f(s') < f(s)$.

In order to speed up the search of solutions,
constraint programming systems alternate search with constraint propagation; each constraint is implemented by a collection of \emph{propagators}. Each propagator's objective is to observe the domains of the variables involved in the constraint and, as soon as a value is removed from the domain of a variable, 
it checks if some values from other domains have become unsupported and can be removed (\emph{pruned}).

A propagator that is selected for execution is said to be \emph{awakened}, or \emph{activated}. 
The execution of a propagator can have one of three possible outcomes. 
In the first case, the propagator realizes that the constraint has no solution, e.g., because a domain becomes empty as a result of deleting some values. 
In this case a failure is returned and the search algorithm undoes the last choice(s) and attempts different assignments.

In the second case, the propagator finds that the constraint is entailed i.e., it is satisfied whatever the values assumed by the variables among those remaining in the domains. In this case, the propagator can be deallocated.

In the third case, neither of the previous two cases has occurred and so the propagator at the end of the execution is suspended waiting for an event to occur.

The \ac{TSP} is defined on a weighted graph $G=(V,E,w)$, where $V$ is a set of nodes ($\vert V \vert = \Npoints$), $E$ is a set of edges, and $w : E \mapsto \mathbb{R}^+$. 
A {\em path} in $G$ is a sequence $\PathFromTo{v_{\seq_0}}{v_{\seq_k}} = v_{\seq_0} e_{\seq_0,\seq_1} v_{\seq_1} \dots e_{\seq_{k-1},\seq_k} v_{\seq_k}$ such that:

\begin{enumerate}[label=(\roman*)]
\item $v_{\seq_0} , v_{\seq_1} , \dots , v_{\seq_k} \in V$
and are all distinct, and
\item $e_{\seq_0,\seq_1} , e_{\seq_1,\seq_2} , 
\dots , e_{\seq_{k-1},\seq_k} \in E$.
\end{enumerate}

Since a path is uniquely identified by the sequence of its nodes (or of its edges) in the proper order, we will often write paths as sequences of nodes to simplify the notation. 
The length of a path \Path\ is the sum of the weights of its edges: $\Length{\Path} = \sum_{i=0}^{k-1} w(e_{s_i,s_{i+1}})$.
A {\em circuit} \Circuit\ is a sequence obtained by appending $e_{\seq_k,\seq_0}$ to a path $\PathFromTo{v_{\seq_0}}{v_{\seq_k}}$.
A circuit that visits each vertex exactly once is often referred as {\em Hamiltonian circuit}. 

In the Euclidean case, let $\PointSet=\{\point_1,\dots,\point_n\}$ be a set of points,
where $\point_i=(x_i,y_i)$.
The graph associated with $\PointSet$ is $G^\PointSet = (\PointSet,E^\PointSet,w^\PointSet)$,
where $E^\PointSet=\{e_{i,j} \equiv (\point_i,\point_j)\mid \point_i,\point_j \in \PointSet, i\neq j\}$
and $w(\point_i,\point_j)= \dist{\point_i}{\point_j}$, where $\distSym$ is the Euclidean distance.

We use $\segment{\point_i}{\point_j}$ to denote the segment in the plane with the extremes $\point_i$ and $\point_j$. Since in the Euclidean case every edge of a graph corresponds to a segment in the plane between the corresponding endpoints, we will often confuse the edge $e_{i, j}$ with the corresponding segment $\segment{\point_i}{\point_j}$.
Also, we will sometimes confuse the index $i \in V$ with the corresponding point $\point_i$ in the plane.
We denote with $\StraightLine{\point_i}{\point_j}$ the (infinite straight) line passing through points $\point_i$ and $\point_j$,
and with $\Angle{\point_i\point_j\point_k}$ the counter-clockwise angle formed by the segments $\segment{\point_i}{\point_j}$ and $\segment{\point_j}{\point_k}$ with vertex in $\point_j$ from $\point_i$ to $\point_k$. 

A {\em crossing} in a path \Path\ is defined as a common point $\point_q \notin \PointSet$ (i.e., internal to two edges) that is shared by two (or more) edges of \Path\, or a common point $\point_r \in \PointSet$ (i.e. an endpoint of some edge) that is shared by three (or more) edges of $\Path$.

In the \ac{CP} literature, three main representations have been devised for modelling the Hamiltonian circuit problem and the TSP: the {\em permutation} representation, the {\em successor} representation, and the {\em set variable} representation~\cite{vanHoeveEtAl_ImprovedWeightedCircuit}.
The last was also extended to the graph representation~\cite{CP_graph,DBLP:journals/corr/abs-1206-3437,SalesmanAndTree}.

In the {\em permutation} representation, \Npoints\ variables $Pos_{i}$ are introduced \( Pos = (Pos_1,\linebreak[1] Pos_2,\linebreak[1] \dots, Pos_\Npoints) \) each with initial domain $V$; variable $Pos_{i}$ represents the $i$-th node that is visited. For example, if $\Npoints = 5$ and $Pos = (3,5,4,2,1)$ the corresponding tour will be $(3,5,4,2,1,3)$. 
The constraint model for the permutation representation includes an {\tt alldifferent}$(Pos)$ constraint~\cite{alldifferent} on the $Pos$ array of variables, that ensures that each node is visited only once.

The {\em set variable} representation is based on a set variable \cite{puget1992pecos, Gervet1993} 
that represents the edges that form the tour.

In the {\em successor} representation, \Npoints\ variables $\Next_i$ are defined \( \Next = (\Next_1,\linebreak[1] \Next_2,\linebreak[1] \dots,\linebreak[1] \Next_\Npoints) \); variable $\Next_i$ represents the node that follows node $i$ in the circuit, and its initial domain is $\{1,\dots,\Npoints\} \setminus \{i\}$. For example, if $\Npoints = 5$ and $\Next = (3,5,4,2,1)$ the corresponding tour will be $(1,3,4,2,5,1)$. 

The constraint model of the successor representation includes, as in the permutation representation, an ${\tt alldifferent}(\Next)$ constraint on the $\Next$ array of variables, 
that
in this case ensures that each node has exactly one incoming edge (degree constraints).
The constraint model for the successor representation also includes a \circuit$(\Next)$~\cite{GlobalConstraintsCHIP,CaseauLaburthe,KayaHooker} constraint (sometimes called {\tt nocycle}) that avoids subtours, i.e., cycles of length less than \Npoints\ (subtour elimination constraints).

In some cases, the constraint model includes, as redundant representation, also a set of $\Prev$ variables: $\Prev_i$
represents the node that precedes node $i$ in the circuit.
Often, constraint models include redundant representations to obtain additional pruning.

Both the {\tt alldifferent} and the \circuit\ constraints are already implemented in most constraint logic programming systems. 
Hereafter, unless otherwise specified, we refer always to the successor representation.

 \section{Related Work} 
\label{sec:related}

As already introduced, the Traveling Salesperson Problem is probably one of the most studied \ac{COP} in the scientific literature. Many approaches proposed for its resolution have been published over time, even if we limit ourselves only to the  \acl{CP} literature.
Usually, when it comes to exploit the objective function to prune the search space, \ac{CP} formulations are not as effective as Integer Programming models, but they are more flexible in dealing with side constraints. 
To date in the \ac{CP} literature, the most efficient solving methods are those that combine constraint propagation with \ac{ILP} techniques. These hybrid methods yield the solving performance of \ac{CP} models, under certain conditions, comparable to that obtained by \ac{ILP} models.

Various works propose to use relaxations of the \ac{TSP} to prune suboptimal branches; the classical relaxations of the \ac{TSP} are the assignment problem and the one-tree relaxation.
\citet{CaseauLaburthe} propose a simple and effective rule for the \circuit\ constraint (called {\tt nocycle}).
The rule consists in finding a path of mandatory edges of length at most $n - 1$ and removing the edge between the two endpoints of the path. 
They also propose to filter 
values based on the objective function, using the assignment-based and the spanning tree relaxations.

\citet{KayaHooker} propose a new filtering rule based on separator graph, able to remove nonhamiltonian edges, for the \circuit\ constraint. 

\citet{explainingCircuit} consider various propagation algorithm for \circuit\ and provide for each so-called explanations in the context of a lazy clause generation solver.

\citet{DBLP:journals/corr/abs-1206-3437} show how properties of the reduced graphs (obtained by introducing a node for each \ac{SCC} of the original graphs) associated to Asymmetric \acp{TSP} can be used to improve the \ac{MST} relaxation.

\citet{PesantGPR98} address the \ac{TSPTW} and exploit the \circuit\ constraint together with the \ac{MST} relaxation.
\citet{FocacciLM_AMAI02,FocacciLM_Informs02} propose a filtering based on reduced costs for optimization constraints, and in particular use the assignment problem and the minimum spanning forest relaxation.

The groundbreaking work in this area is that by
\citet{vanHoeveEtAl_ImprovedWeightedCircuit}: it
is the first in which a \ac{CP} model was able to solve large \ac{TSP} instances.
In their work, they use a variety of techniques.
They propose an implementation of the weighted circuit constraint that includes the \citet{HeldKarp} scheme,
iterates a Lagrangian relaxation to obtain a high-quality one-tree, and uses it to
remove edges  similarly to reduced costs filtering.
It also identifies as mandatory those edges that, if removed, would increase the current lower bound over the quality of the incumbent solution.
To find quickly a solution, they first run the Lin-Kernighan-Helsgaun
algorithm~\citep{LinKernighan,Helsgaun00}.
In the experiments with asymmetric \acp{TSP}, they also use additive bounding~\citep{AdditiveBoundingTSP} to combine both the 1-tree and the assignment problem relaxations.

In \citet{SalesmanAndTree}, the authors further improve by casting the problem in CP(Graph) and by means of improved search heuristics (i.e., Last Conflict heuristic).

More recently, \citet{Boudreault2021} extend the work of \citet{vanHoeveEtAl_ImprovedWeightedCircuit} by proposing an enhanced CP-based Lagrangian relaxation approach. This method improves the filtering capabilities of the weighted circuit constraint by locally adjusting the Lagrangian multipliers to maximize the number of filtered values. They also develop two novel algorithms, SIMPLE and $\alpha$-SETS, which provide stronger filtering and more efficient pruning of the search space.

\citet{isoart2019integration} design a propagator based on the search of $k$-cutsets. The combination of this constraint with the \ac{WCC} constraint has resulted in a significant reduction in the computation time.

We have already pointed out how from the computational complexity point of view,
both the Metric and the Euclidean \ac{TSP} are \NPH~\citep{TSPeuclideoNPhard}; nonetheless, differently from the general \ac{TSP}, the Euclidean \ac{TSP} admits a \ac{PTAS}~\citep{Arora_PTAS}.
A \ac{PTAS} is an algorithm that, given an instance of an optimization problem and fixed a parameter $\varepsilon > 0$, produces in polynomial time a solution within a factor $1 + \varepsilon$ of the optimum. Note that the computation time can still be an exponential function of ${1}/{\epsilon}$.  
Notwithstanding these theoretically important results, solvers addressing the general \ac{TSP} (e.g., Concorde) are faster in practice. 

Finally, concerning the \ac{GTSP}, most approaches in the literature are based on \ac{ILP} or metaheuristics. To the best of our knowledge, the use of \ac{CP} for the \ac{GTSP} has not previously been explored.

Fischetti et al. studied the face structure of the polytope resulting from an \ac{ILP} formulation of the problem~\citep{DBLP:journals/networks/FischettiGT95} and proposed exact and heuristic separation algorithms for several classes of face-defining inequalities, including comb inequalities~\citep{DBLP:journals/ior/FischettiGT97}. These algorithms were used within a branch-and-cut framework. \citet{LAPORTE1987} addressed the problem using a subtour-elimination constraint relaxation algorithm. They also considered the asymmetric case, which they solved using a modified version of the branch-and-bound algorithm for the asymmetric \ac{TSP} proposed by \citet{CarpanetoToth1980}.

More recent works have focused on approximate methods. \citet{DBLP:conf/hm/PopMS10} proposed a hybrid genetic algorithm for the \ac{GTSP} and compared it with the random-key genetic algorithm of \citet{DBLP:journals/eor/SnyderD06} and the memetic algorithm of \citet{DBLP:journals/nc/GutinK10}. Their results showed competitive solution quality and computation time.

\section{Avoiding crossings}
\label{sec:nocrossing}

Among the geometric information that can be exploited to solve Euclidean TSPs, a well-known result in the literature is that the optimal solution of a metric TSP (thus, also of a Euclidean TSP) in the plane cannot include two edges that cross each other (see Figure~\ref{fig:incrocio}). 
We say that a set of points are all aligned if they lie on the same straight line.

\begin{theorem}[\citet{Flood}]
\label{thm:nocrosses}
Let $\Circuit^*$ be an optimal tour of a metric TSP.
Then, for each $e_{i,j}, e_{k,l} \in \Circuit^*$ such that $\{i,j,k,l\}$ are all different and not all aligned, the segments $\segment{\point_i}{\point_j} \cap \segment{\point_k}{\point_l} = \emptyset$.
\end{theorem}

\begin{proof}
By contradiction, suppose that the optimal tour $\Circuit^*$ contains two segments  $e_{i,j}, e_{k,l}$ such that $\{i,j,k,l\}$ are all different and $\segment{\point_i}{\point_j} \cap \segment{\point_k}{\point_l} = \{Q\}$.

Without loss of generality, let $\Circuit^* = \point_i e_{i,j} \point_j \PathFromTo{j}{k} \point_k e_{k,l} \point_l \PathFromTo{l}{i} \point_i$. Consider now the tour $\Circuit^\dag = \point_i e_{i,k} \point_k \PathFromTo{k}{j} \point_j e_{j,l} \point_l \PathFromTo{l}{i} \point_i$,
where $\PathFromTo{k}{j}$ is the path $\PathFromTo{j}{k}$ reversed.

The difference $\Length{\Circuit^*}-\Length{\Circuit^\dag} = \weight{e_{i,j}} + \weight{e_{k,l}} - \weight{e_{i,k}} - \weight{e_{j,l}} 
= (\dist{\point_i}{Q} + \dist{Q}{\point_j}) + (\dist{\point_k}{Q} + \dist{Q}{\point_l}) - \dist{\point_i}{\point_k} - \dist{\point_j}{\point_l}
\geq 0$ applying the triangle inequality to the triangles $(\point_i,\point_k,Q)$ and $(\point_j,\point_l,Q)$.

Considering that the points $(\point_i,\point_k,Q)$ and $(\point_j,\point_l,Q)$ are not aligned, the inequality becomes strict, so 
$\Length{\Circuit^\dag} < \Length{\Circuit^*}$, that contradicts the fact that $\Circuit^*$ is optimal.
\end{proof}

\begin{figure}[t]
\centering
\includegraphics[width=.6\textwidth]{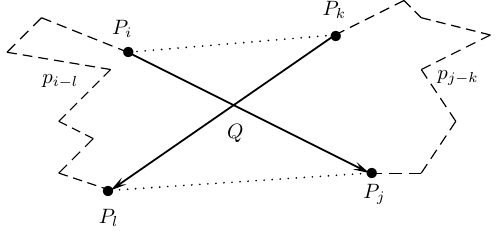}

 \caption{\label{fig:incrocio} A self-crossing circuit.  Theorem~\ref{thm:nocrosses} suggests that instead of taking \segment{\point_i}{\point_j} and \segment{\point_k}{\point_l}, a shorter tour chooses the dotted edges \segment{\point_i}{\point_k} and \segment{\point_j}{\point_l}}
\end{figure}

\subsection{Introducing nocrossing constraint}

From Theorem~\ref{thm:nocrosses} follows that during the search for an optimal TSP it is possible to avoid those Hamiltonian paths that include crossing edges; this can be seen as a \emph{dominance rule}~\citep{JOUGLET2011433}. Based on this classical result, we propose the \nocrossing\ constraint, which imposes that two segments in the TSP should not cross each other. In the successor representation, it is defined as follows:

\begin{equation}
\begin{gathered}
\nocrossing_{i,j}(\Next_i,\Next_j) = \\
= \left\lbrace \left(n_i, n_j\right) \in \Dom{\Next_i} \times \Dom{\Next_j} \vert \left(\segment{\point_i}{\point_{n_i}} \cap \segment{\point_j}{\point_{n_j}} \right)
\subset \{\point_i,\point_j\} \right\rbrace
\label{eq:nocrossing_definition}
\end{gathered}
\end{equation}

Since the successor representation includes the ${\tt alldifferent}(\Next)$ constraint, the variables $Next_i$ and $Next_j$ cannot assume the same value so the only type of intersection allowed by the \nocrossing\ constraint, presented in Equation~\ref{eq:nocrossing_definition}, is on one of the extremes i.e. $\point_{Next_i} \equiv \point_j$ or $\point_{Next_j} \equiv \point_i$.

The \nocrossing\ constraint is a \emph{binary} constraint, i.e., it involves exactly two decision variables, 
and it is also a \emph{redundant} constraint as its inclusion in the model does not change the set of solutions.

For a TSP instance with $n$ nodes, we introduce $\frac{n(n-1)}{2}$
constraints, one for each pair of nodes in the graph. The \nocrossing\ constraint can be implemented through a pair of propagators: 
one propagator removes values from the domain of $\Next_i$ based on the values in the domain of $\Next_j$, while the other propagator propagates changes in the other direction.

The \nocrossing\ constraint propagator could be implemented naively, for example using the table constraint~\citep{tableConstraintNengFa,DBLP:conf/ijcai/BessiereR97, DBLP:conf/aaai/GentJMN07, DBLP:conf/cp/Lecoutre08, DBLP:conf/cp/KatsirelosW07} or the {\tt propia} library~\citep{propia}. A table constraint consists of a table (usually a list of tuples) of values that the involved variables must, or must not, assume. However, these implementations would be inefficient, because the constraint wakes up most of the time without being able to propagate. In addition, with the table constraint one should initially compute large tables, containing, for all pairs of edges in the graph, if they cross or not.

A naive propagator also tends to wake up every time a value is removed from the domain of $\Next_i$ in order to remove inconsistent values from the domain of $\Next_j$. Propagating such constraint, with a naive propagator, would have the usual cost of arc consistency~\citep{AC3} propagation for a single constraint of $O(d^2)$ (if $d$ is the size of the domains) in each activation of the constraint.

From the definition of arc consistency and Equation~\ref{eq:nocrossing_definition}, a value $v \in \Dom{\Next_j}$ can be removed from $\Dom{\Next_j}$ only if \segment{\point_j}{\point_v} intersects all
possible segments originating from $\point_i$.
A necessary condition for this is that all segments originating from $\point_i$ lie on the same half-plane with respect to the line \StraightLine{\point_i}{\point_j}.

\begin{theorem}
\label{thm:same_side}
Let $o, u \in \Dom{\Next_i}$ such that $\point_o$ and $\point_u$ do not lie on the line $l$ passing through
$\point_i$ and $\point_j$ and
are in different half-planes with respect to the line $l$.
Then, any value $k\in \Dom{\Next_j}$  such that $\point_k \not\in l$
is arc consistent with respect to the constraint $\nocrossing_{i,j}(\Next_i,k)$.
\end{theorem}
\begin{proof}
By contradiction, suppose there exists $k\in \Dom{\Next_j}$ that is inconsistent with the constraint $\nocrossing_{i,j}(\Next_i,k)$.  
Since it is inconsistent, the segment $\segment{\point_j}{\point_k}$ must cross all the segments $\segment{\point_i}{\point_z}$ such that $z\in \Dom{\Next_i}$, so in particular it crosses both $\segment{\point_i}{\point_o}$ and $\segment{\point_i}{\point_u}$.
But the intersection $I_o\equiv \segment{\point_i}{\point_o} \cap \segment{\point_j}{\point_k}$ lies on a different half-plane from the one that hosts $I_u\equiv \segment{\point_i}{\point_u} \cap \segment{\point_j}{\point_k}$,
so $\point_k$ lies in both half-planes, meaning that $\point_k \in l$: absurd.
\end{proof}

Theorem~\ref{thm:same_side} does not cover the case in which one of the points $\point_o$, $\point_u$ lies on the line $\StraightLine{\point_i}{\point_j}$. The following proposition deals with the cases where three points are aligned.

\begin{proposition}
\label{thm:three_points_aligned}
Given a graph $G$ whose nodes do not lie all on the same line;
let $a$, $b$ and $c$ be three nodes of  $G$ such that $\point_c \in \segment{\point_a}{\point_b}$.
Then, segment $\segment{\point_a}{\point_b}$ is not in the optimal TSP. 
\end{proposition}
\begin{proof}
The proof is similar to that of Theorem~\ref{thm:nocrosses} as this proposition is no more than a special case.
Let $\point_c$ be the point that lies on segment $\segment{\point_a}{\point_b}$, by contradiction suppose that segment $\segment{\point_a}{\point_b}$ appears in the optimal TSP.

The optimal TSP $\Circuit^*$ contains two points: one immediately preceding $\point_c$, which we will denote by $\point_{c^{'}}$, and one immediately following, denoted by $\point_{c^{''}}$.
Let $\Circuit^* = \point_a e_{a,b} \point_b \PathFromTo{b}{c^{'}} \point_{c^{'}} e_{c^{'},c} \point_c e_{c,c^{''}} \point_{c^{''}} \PathFromTo{c^{''}}{a}$. 

Now, point $\point_c$ must be on segment $\segment{\point_{c^{'}}}{\point_{c^{''}}}$, if the points were not aligned it would be better to take segments $\segment{\point_{c^{'}}}{\point_{c^{''}}}$, $\segment{\point_a}{\point_c}$ and $\segment{\point_c}{\point_b}$.  

Consider now the tour $\Circuit^\dag = \point_a e_{a,c^{'}} \point_{c^{'}} \PathFromTo{c^{'}}{b} \point_b e_{b,c} \point_c e_{c, c^{''}} \point_{c^{''}} \PathFromTo{c^{''}}{a}$. 
The difference $\Length{\Circuit^*} - \Length{\Circuit^\dag} = \weight{e_{a,b}} + \weight{e_{c^{'},c}} - \weight{e_{a,c^{'}}} - \weight{e_{b,c}}$.
At the beginning, we assumed that $\point_c$ is a point on segment $\segment{\point_a}{\point_b}$, so we can rewrite $\weight{e_{a,b}}$ as $\weight{e_{a,c}} + \weight{e_{c,b}}$ obtaining $\Length{\Circuit^*} - \Length{\Circuit^\dag} = \weight{e_{a,c}} + \weight{e_{c^{'},c}} - \weight{e_{a,c^{'}}} = \dist{\point_a}{\point_c} + \dist{\point_{c^{'}}}{\point_c} - \dist{\point_{a}}{\point_{c^{'}}}$ which is positive by the triangle inequality, contradicting the fact that segment $\segment{\point_a}{\point_b}$ can be part of $\Circuit^*$.

\end{proof}

Given Proposition~\ref{thm:three_points_aligned}, it is possible to remove from the domain of each variable $\Next_a$ all the values $b$ such that the segment $\segment{\point_a}{\point_b}$ contains another node of the graph $G$.  For the remainder of this article, we will assume that this pre-processing step has been performed before the search begins; this assumption simplifies the following discussion.

Algorithm~\ref{alg:nocrossing_propagator} sketches the algorithm of a propagator for the \nocrossing\ constraint; it is awakened when the domain of variable $\Next_i$ is reduced, and it performs propagation to possibly reduce the domain of $\Next_j$; to fully implement the constraint, another symmetric propagator would be imposed in the reverse direction (from $\Next_j$ to $\Next_i$).

\begin{algorithm}[tbp]
\small
\caption{\nocrossing\ propagator \label{alg:nocrossing_propagator}}
\begin{algorithmic}[1]

\Function{nocrossing\_propagator}{$i,\Next_i,j,\Next_j$}\label{alg:nocrossing:line_phase_one}
\State	$Under \gets$ select one element in \Dom{\Next_i} s.t. $\point_{Under}$  is under the line $\StraightLine{\point_i}{\point_j}$
\If {there is no such element}
\State \Call {nocrossing\_propagator\_phase\_2}{$i,\Next_i,j,\Next_j$}
\Else\ $Over \gets$ select one element in $\Dom{\Next_i}$ s.t. $\point_{Over}$ is over the line \StraightLine{\point_i}{\point_j}
	\If {there is no such element}
	\State \Call {nocrossing\_propagator\_phase\_2}{$i,\Next_i,j,\Next_j$}
	\Else\ suspend waiting for either $Over$ or $Under$ to be removed from $\Dom{\Next_i}$
	\EndIf
\EndIf
\EndFunction 
\label{alg:nocrossing:line_end_phase_one}

\Function{nocrossing\_propagator\_phase\_2}{$i,\Next_i,j,\Next_j$} \label{alg:nocrossing:line_phase_two}
\State $\minAlpha \gets \min \{\alpha_x \mid x \in \Dom{\Next_i}\}$ \label{alg:nocrossing:line_compute_minAlpha}
\State Let $x^i_{\minAlpha}$ be the value in $\Dom{\Next_i}$ corresponding to \minAlpha
\State $\maxBeta \gets \max \{\beta_x \mid x \in \Dom{\Next_i}\}$ \label{alg:nocrossing:line_compute_maxBeta}
\State Let $x^i_{\maxBeta}$ be the value in $\Dom{\Next_i}$ corresponding to \maxBeta
\ForAll {$y^j \in \Dom{\Next_j} \mbox{ s.t. } \alpha_{y^j} < \minAlpha$}  \label{alg:nocrossing:line_forall_less_minAlpha}
\If {$\beta_{y^j}>\maxBeta$} \label{alg:nocrossing:line_check_beta}
	\State remove $y^j$ from $\Dom{\Next_j}$
	\EndIf
\EndFor \label{alg:nocrossing:line_end_forall_less_minAlpha}
\If {$|\Dom{\Next_i}| > 1 \ \land\ |\Dom{\Next_j}| > 1$}
	\State suspend waiting for either $x^i_{\minAlpha}$ or $x^i_{\maxBeta}$  to be removed from $\Dom{\Next_i}$
\EndIf

\EndFunction
\end{algorithmic}
\end{algorithm}

As long as the value $j$ is contained in the domain of the variable $\Next_i$, according to Theorem~\ref{thm:nocrosses}, no propagation can take place because the point $\point_j$ lies on the same straight line as the point $\point_i$; for this reason, the propagator described in Algorithm~\ref{alg:nocrossing_propagator} is imposed only when the value $j$ is removed from $\Dom{\Next_i}$ .

In the first phase, lines \ref{alg:nocrossing:line_phase_one}-\ref{alg:nocrossing:line_end_phase_one} of Algorithm~\ref{alg:nocrossing_propagator}, the propagator is suspended 
and waits for all elements in the domain of $\Next_j$ to lie in the same half-plane with respect to \StraightLine{\point_i}{\point_j}.
To do so, one element $Over \in \Dom{\Next_j}$ and one $Under \in \Dom{\Next_j}$
that lie, respectively, over and under the line $\StraightLine{\point_i}{\point_j}$ are selected.
If one of them does not exist, all possible segments originating from $\point_{j}$ lie on the same half-plane and the control passes to the next phase; otherwise, the propagator suspends waiting for one of the two $Over$ and $Under$ elements to be removed from $\Dom{\Next_j}$.
This strategy mimics, in a sense, the idea of watched literals originally proposed in SAT solvers~\citep{CHAFF} and currently used  in several solvers,
and significantly reduces the number of activations of the propagator.

Checking if a point $P_a$ lies in the half-plane under or over the line \StraightLine{\point_b}{\point_c} amounts to
check the sign of the (projection on the $z$ axis) of the cross product $(\point_a-\point_b) \times (\point_c-\point_b)$.

When all segments in the domain of the variable $\Next_i$ lie on the same half-plane with respect to the line $\StraightLine{\point_i}{\point_j}$, we should, for each element $x \in \Dom{Next_j}$, check if the segment \segment{\point_j}{\point_x} crosses 
all segments exiting from the point $\point_i$. Naively performing the crossing check would require a number $O(n^2)$ of operations  each time the propagator wakes up (typically an exponential number of times while exploring the search space).

In order to reduce the number of crossing checks, it is possible to sort the elements $\point_x \in \Dom{Next_i}$ from smallest (denoted by the symbol $\minAlpha$) to largest according to the $\Angle{\point_i\point_j\point_x}$ angle.
Figure~\ref{fig:figura_fase2} shows that a segment $\segment{\point_j}{\point_t}$ ($\point_t \in \Dom{\Next_j}$) could intersect all segments originating from $\point_i$ only if the angle $\Angle{\point_i\point_j\point_t}$ is smaller than $\minAlpha$. In addition, if such a segment does not exist then there is no value to be removed from the $\Dom{\Next_j}$ and we can suspend the propagator waiting for the value corresponding to $\minAlpha$ to be removed from $\Dom{\Next_i}$.
This turns out to be a necessary condition for the \nocrossing\ propagator to remove values from the domain of $\Next_j$
and is further detailed in Theorem~\ref{thm:necessary_phase2}.

\begin{theorem}
\label{thm:necessary_phase2}
For each $k \in \Dom{\Next_i} \cup \Dom{\Next_j}$, let $\alpha_k=\Angle{\point_i\point_j\point_k}$. 
If all the elements $q \in \Dom{\Next_i}$ lie on the same half-plane with respect to the line $\StraightLine{\point_i}{\point_j}$,
then 
a necessary condition for a segment $\segment{\point_j}{\point_t}$ originating from $\point_j$ and reaching an element $t \in \Dom{\Next_j}$
to cross all segments $\segment{\point_i}{\point_q}$ originating from $\point_i$ is that 
$\alpha_t \leq \minAlpha$,
where 
$\minAlpha = \min \{ \alpha_q \mid q \in \Dom{\Next_i}
\}$.
\end{theorem}

\begin{proof}
Consider a coordinate system centered into $\point_j$, with the abscissa pointing toward $\point_i$ and such that all the points in $\Dom{\Next_i}$ have non-negative ordinate (see Figure~\ref{fig:figura_fase2}).

By contradiction, suppose that 
$\alpha_t > \minAlpha$; 
we prove that there is a segment originating from $\point_i$ that does not intersect with $\segment{\point_j}{\point_t}$.
Let $z \in \Dom{\Next_i}$ such that $\minAlpha = \Angle{\point_i\point_j\point_z}$.

In polar coordinates, the segment \segment{\point_i}{\point_z} is seen from $\point_j$ with angles between 0 and $\minAlpha$.
All the points on the segment \segment{\point_j}{\point_t} are seen under the angle $\alpha_t$.
Since $\alpha_t > \minAlpha$, there is no intersection between $\segment{\point_i}{\point_z}$ and $\segment{\point_j}{\point_t}$.
\end{proof}

\begin{figure}[tbp]
\centering
\includegraphics[width=\textwidth]{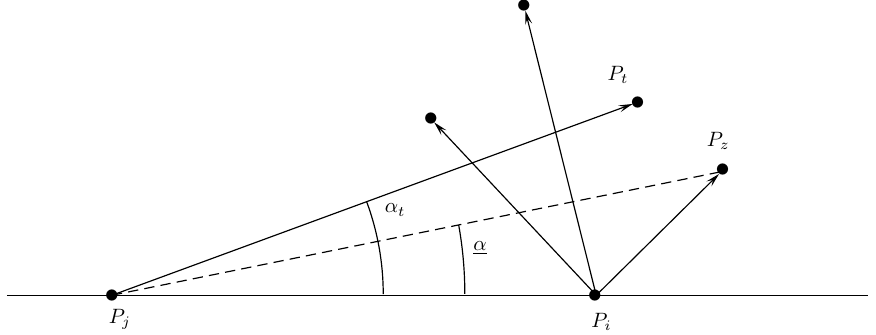}
\caption{\label{fig:figura_fase2} An arrow from $\point_x$ to $\point_y$ means that $y \in \Dom{\Next_x}$. Dashed lines are plotted to show the angles.}
\end{figure}

The condition of Theorem~\ref{thm:necessary_phase2} is only necessary,
a counterexample is shown in Figure~\ref{fig:figura_fase2counterexample},
where segment $\segment{\point_j}{\point_t}$ does not cross all segments that leave the node $\point_i$. 

\begin{figure}[tbp]
\centering
\includegraphics[width=\textwidth]{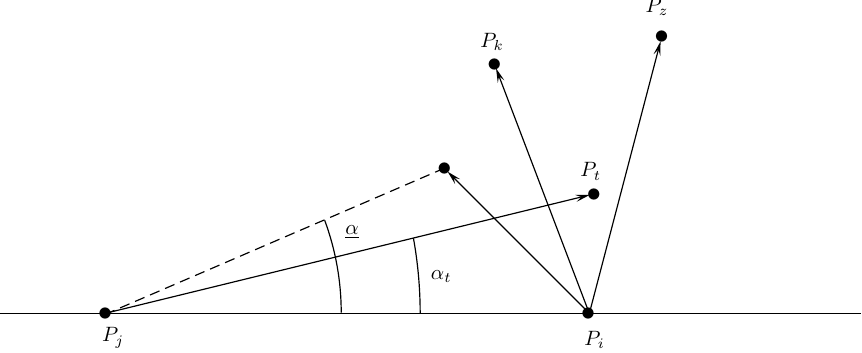}
\caption{\label{fig:figura_fase2counterexample} The condition in Thm~\ref{thm:necessary_phase2} is not sufficient: $\alpha_t < \minAlpha$ but $\segment{\point_j}{\point_t}$ does not cross all segments exiting from $\point_i$.}
\end{figure}

The following theorem provides a sufficient condition for a segment to cross all segments originating from $\point_i$.

\begin{theorem}
\label{thm:sufficient_phase2}
For each $k \in \Dom{\Next_i} \cup \Dom{\Next_j}$, let $\beta_k = \Angle{\point_k\point_i\point_j}$.
Assume all the elements $q \in \Dom{\Next_i}$ lie on the same half-plane with respect to the line $\StraightLine{\point_i}{\point_j}$.
Suppose there exists $t\in \Dom{\Next_j}$ such that $\alpha_t < \minAlpha$.

Then a sufficient condition for segment $\segment{\point_j}{\point_t}$ to cross all segments $\segment{\point_i}{\point_q}$
such that $q \in \Dom{\Next_i}$ is that $\beta_t > \maxBeta$, where $\maxBeta = \max\{\beta_q \mid q \in \Dom{\Next_i} \}$.
\end{theorem}

\begin{proof}
By contradiction, suppose $\exists k\in \Dom{\Next_i}$ such that $\segment{\point_i}{\point_k}$ does not intersect $\segment{\point_j}{\point_t}$ (Figure~\ref{fig:sufficient_phase2}).
Since $\alpha_t < \minAlpha \leq \alpha_k$, $\point_k$ and $\point_i$ lie on different sides of the ray $\overrightarrow{\point_j\point_t}$.

In order not to have a crossing between $\segment{\point_i}{\point_k}$ and $\segment{\point_j}{\point_t}$, 
both $\point_j$ and $\point_t$ must lie on the same half-plane with respect to $\overrightarrow{\point_i\point_k}$.
Since $\beta_j=0$, and all points in the domain of $\Next_i$ are above the $x$ axis,
$\beta_j < \beta_k$, thus to be on the same half-plane, also $\beta_t < \beta_k$.
But $\beta_t > \maxBeta \geq \beta_k$: contradiction.
\end{proof}

\begin{figure}[tbp]
\centering
\includegraphics[width=\textwidth]{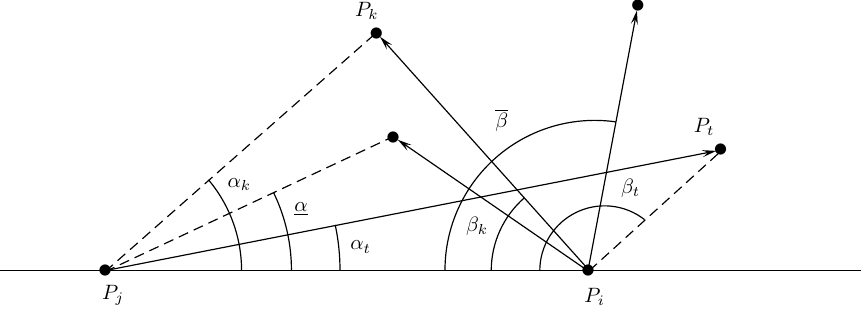}
\caption{\label{fig:sufficient_phase2} An arrow from $\point_x$ to $\point_y$ means that $y \in \Dom{\Next_x}$. Dashed lines are plotted to show the angles.}
\end{figure}

Thanks to Theorems~\ref{thm:necessary_phase2} and \ref{thm:sufficient_phase2} we can get a better complexity compared to the naive algorithm. Note that all angles can be pre-computed before starting the search, which avoids the computation of trigonometric functions during the search. It is also possible to pre-compute the elements in the (initial) domains of the two variables, sorted according to their angle $\alpha$, and then store them in a linked list. 
In this way, the computation of the minimum in line~\ref{alg:nocrossing:line_compute_minAlpha} of Algorithm~\ref{alg:nocrossing_propagator} is equivalent to finding the first element in the linked list that belongs to $\Dom{\Next_i} $; assuming that the domain membership is checked in constant time, the minimum on line~\ref{alg:nocrossing:line_compute_minAlpha} can be found in $O(d)$ amortized time on one branch of the search tree (if $d$ is the number of elements in the domain). 
In the same way, the elements in $\Dom{\Next_i}$ can be sorted by their angle $\beta$, and
line~\ref{alg:nocrossing:line_compute_maxBeta} is also executed in $O(d)$ amortized time on a branch of the search tree. The loop in lines~\ref{alg:nocrossing:line_forall_less_minAlpha} - \ref{alg:nocrossing:line_end_forall_less_minAlpha} searches the linked list and stops as soon as an element is found with an angle greater than or equal to $\minAlpha$; assuming that the removal of a domain element takes place in constant time, and since the comparison on line~\ref{alg:nocrossing:line_check_beta} takes a constant time, the entire loop has $O(d)$ complexity for each activation of the propagator of phase~2 (to be compared with the $ O(d^2)$ of the naive propagator). Since the propagator is activated when at least one element from $\Dom{\Next_i} $ is removed, this propagator is woken up at most $O(d)$ times in a branch of the search tree, which gives phase~2 a general complexity of $O(d^2)$ amortized time in one branch of the search tree.

\subsection{Convex hull reasoning and clockwise constraint}
\label{sec:convex}
A useful consequence of Theorem~\ref{thm:nocrosses} is given in the following Corollary~\ref{thm_c:convexhull} and is based on the concept of convex hull.
The {\em convex hull} $\ConvexHull{\PointSet}$ of a set of points $\PointSet$ in a Euclidean space is the minimum convex set containing all the points.
In the plane it corresponds to a convex polygon, and it is completely defined by its vertices.

\begin{corollary}[\citet{DBLP:journals/ipl/DeinekoDR94}]
\label{thm_c:convexhull}
``Assuming that not all cities lie on one line, an optimal tour has the property that the cities on the boundary
of the convex hull of the cities are visited in their cyclic order''.
\end{corollary}

We denote with $\HullSet{\PointSet} = \langle\Hull_0, \Hull_1, \dots, \Hull_{\NHull-1}\rangle$ the sequence of vertexes on the boundary of the convex hull in clockwise order. 
We also define a successor function $S_{\HullSet{\PointSet}}(\Hull_i) = \Hull_{(i+1) \mod \NHull}$ to  denote the successor of the point $\Hull_i$ in the sequence $\HullSet{\PointSet}$. 
Whenever it is clear from the context, we omit the sequence and indicate the successor of $\Hull_i$ with $\HSucc{\Hull_i}{\PointSet}$ instead of $S_{\HullSet{\PointSet}}(\Hull_i)$.
To compute $\HullSet{\PointSet}$, we used the widely known Andrew's monotone chain algorithm~\citep{AndrewMonotoneChain}, with complexity $O(\Npoints \log \Npoints)$. The algorithm calculates, with complexity $O(\Npoints)$, the upper and lower hull, which are then combined to form $\HullSet{\PointSet}$. However, this requires that the points are first sorted with respect to their $x$-coordinates, and also with respect to their $y$-coordinates in case of a tie, hence the complexity already mentioned.

To simplify the exposition, the following pruning is presented in the context of symmetry breaking constraints, although similar reasoning might be performed also while not breaking symmetries.

In the successor representation, the same \ac{TSP} can be represented by two symmetrical solutions that differ just for the order (clockwise or counter-clockwise) in which the nodes are visited.
One way to break this symmetry is to fix one direction (e.g., clockwise);
in such a case, the convex hull reasoning is an efficient way to impose the clockwise order.

Based on this classical cyclic-order property, we devised three ways to exploit the information about the hull for propagation.

The simplest is to impose that the successor of a convex hull vertex cannot be another vertex member of $\HullSetP$ except
for the one that immediately follows it, see Equation~\ref{eq:successor_constraint}.
\begin{equation}
\label{eq:successor_constraint}
\forall i \in [0,\NHull-1], \ 
\Dom{\Next_{\Hull_i}} \cap \HullSetP \subseteq \{\HSucc{\Hull_i}{\PointSet}\}
\end{equation}

Equation~\ref{eq:successor_constraint} represents a set of unary constraints, which are in practice equivalent to reducing the initial domain of the variables, so no overhead is produced during the search (see Figure~\ref{fig:hull0}). 

\begin{figure}[tbp]
\centering
\includegraphics[width=.5\textwidth]{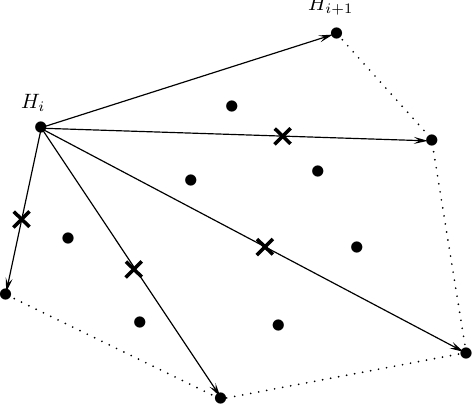}
 \caption{\label{fig:hull0} The successor of a convex hull vertex cannot be another vertex on the boundary of the convex hull except for the one that immediately follows it.}
\end{figure}

The second way is reasoning on the angle formed by the incoming and the outgoing arcs in hull vertexes:
in order to visit nodes in a clockwise order, the angle between the incoming edge and the outgoing edge
of $\Hull_i \equiv \point_h$ cannot be positive (it must be between $-\pi$ and 0) or, stated otherwise, it must correspond to a right turn (see Figure~\ref{fig:hull1}).

\begin{figure}[t]
\centering
\includegraphics[width=.9\textwidth]{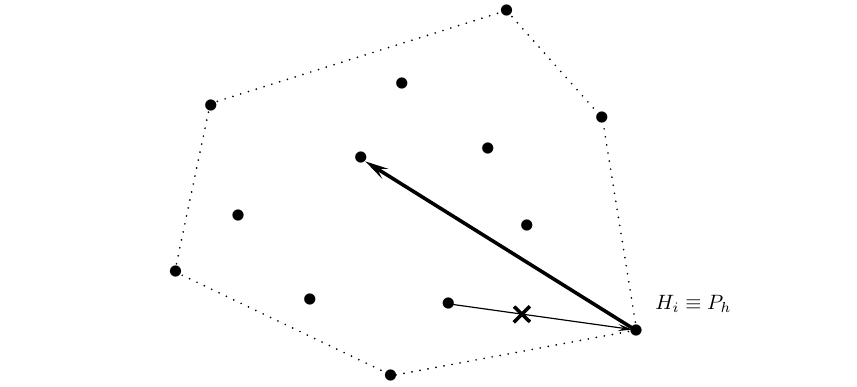}
 \caption{\label{fig:hull1} In order to visit nodes in a clockwise order, the angle between the incoming edge and the outgoing edge of a convex hull vertex cannot be positive (it must be between $-\pi$ and 0).}
\end{figure}

The easiest implementation consists of waiting until $\Next_h$ becomes instantiated with a value;
when the arc outgoing from $\point_h$ is fixed, the value $h$ is removed from the domain of all other variables $\Next_i$ such that
the angle $\Angle{\point_i \point_h \point_{\Next_h}}$ would correspond to a left turn.

If the constraint model also contains the \Prev\ variables, the angle formed by the outgoing edge and any reference direction must be smaller than the angle that the incoming edge forms with the same reference direction.
This produces a propagator in the same spirit of the classical {\em less-than} propagator: simply
compute the minimum angle in $\Dom{\Next_h}$ and remove from $\Dom{\Prev_h}$ the elements associated with a smaller (or equal) angle
(Algorithm~\ref{alg:clockwise_angle_propagator}). A symmetric propagator takes care of the opposite direction
(from $\Prev_h$ to $\Next_h$).
Again, note that all angles are pre-computed before search, and the search for the minimum takes $O(d)$ amortized time over one branch of the search tree.

\begin{algorithm}[htbp]
\small
\caption{clockwise\_angle\_propagator \label{alg:clockwise_angle_propagator}}
\begin{algorithmic}[1]
\Require $\point_h$ to be on the boundary of the convex hull
\Function{clockwise\_angle\_propagator}{$h,\Next,\Prev$}
\State $m \gets \min \{\alpha_v \mid v \in \Dom{\Next_h} \}$
\State $\Dom{\Prev_h} \gets \Dom{\Prev_h} \setminus \{i \mid \alpha_i \leq m\}$
\If {$\Next_h$ and $\Prev_h$ are not instantiated with a value}
	\State suspend until $m$ is removed from $\Dom{\Next_h}$
\EndIf
\EndFunction
\end{algorithmic}
\end{algorithm}

\begin{figure}[t]
\centering
\includegraphics[width=.5\textwidth]{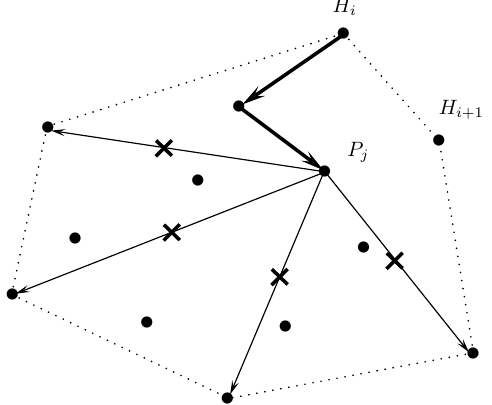}
 \caption{\label{fig:hull2} Each path originating from a convex hull vertex cannot reach any convex hull vertex except for the one immediately following it. Implementation is inspired by the \circuit\ constraint~\citep{CaseauLaburthe} but performs more powerful pruning.}
\end{figure}

The third way results from stating that any path starting from a convex hull vertex cannot reach any other convex hull vertex except the one that directly follows it. Put it more precisely, 
each vertex in a path originating from a point $\Hull_i$ cannot reach any vertex in $\HullSetP$ except for $\HSucc{\Hull_i}{\PointSet}$ (see Figure~\ref{fig:hull2}).
The propagator is imposed for each pair $(\Hull_i,\HSucc{\Hull_i}{\PointSet})$.
The implementation of this propagator is inspired by the \circuit\ constraint~\citep{CaseauLaburthe}
but performs more powerful pruning thanks to the convex hull reasoning.
If a partial path has been defined starting from $\Hull_i$ up to a node $j$, and 
such path does not contain vertex $\HSucc{\Hull_i}{\PointSet}$,
then the variable $\Next_j$ cannot take any value in $\HullSetP$ except for $\HSucc{\Hull_i}{\PointSet}$ (Algorithm~\ref{alg:hull_path_propagator}).
If the partial path reaches the next vertex of the convex hull, the constraint is entailed. In the propagator for the \circuit\ constraint~\citep{CaseauLaburthe}, instead, from the domain of the variable $\Next_j$ only the initial value
(in our case, $\Hull_i$) of the path is removed.
For each vertex $\Hull_i \in \HullSetP$, the propagator is initially invoked with $Start = \Hull_i$ and
$End = \HSucc{\Hull_i}{\PointSet}$. In the recursive calls, $End$ remains unchanged, while $Start$ advances along the currently instantiated path.

\begin{algorithm}[htbp]
\small
\caption{hull\_path\_propagator \label{alg:hull_path_propagator}}
\begin{algorithmic}[1]
\Require $\point_{End}$ to be on the boundary of the convex hull.
\Function{hull\_path\_propagator}{$Start,End,\HullSetP,\Next$}
\If {$Start$ == $End$}
	\State return $true$ \label{alg:hull_path_propagator:line_entailed}
\EndIf
\If {$\Next_{Start}$ is instantiated with a value}
	\State \Call {hull\_path\_propagator}{$\Next_{Start},End,\HullSetP,\Next$}
\Else
	\State remove $(\HullSetP \setminus \{End\})$ from $\Dom{\Next_{Start}}$
	\State suspend waiting for $\Next_{Start}$ to become instantiated with a value
\EndIf
\EndFunction
\end{algorithmic}
\end{algorithm}

\subsubsection{Extension of the convex hull reasoning}

Now that we have propagators that exploit the knowledge about the convex hull,
we wish to extend their applicability to the points in the interior 
of the hull as well.

One consequence of the absence of crossings is that the optimal \ac{TSP} is a
simple polygon, which is a closed polygonal chain of line segments that do not cross each other, and it divides the plane into exactly two areas: an {\em internal} and an {\em external} area.

Now imagine cutting the optimal \ac{TSP} with two vertical lines (parallel to the $y$ axis):
the stripe between the two lines will contain alternate internal and external areas.
The borders (i.e., the parts of the circuit inside the stripe) will be visited alternately from left to right (or clockwise) and from right to left (or counter-clockwise).
To exploit this informal intuition for pruning, 
we provide the following theorem:

\begin{theorem}
\label{thm:HullRicorsiva}
Suppose that (e.g., during search) a partial path $\PathFromTo{s}{e}$ has been defined, starting in node $s$ and ending in node $e$.
Consider the polygon $Q$ delimited by such path and by the segment $\segment{\point_s}{\point_e}$, and suppose that such polygon is a simple polygon,
i.e., no two edges intersect.
Suppose that the partial path $\PathFromTo{s}{e}$ touches its vertexes in clockwise order.

Let $F \subset V$ be the set of nodes whose corresponding points lie in the interior of polygon $Q$, $I = F \cup \{s,e\}$ and let $\HullSetI =\langle\Hull^I_0\equiv e,\Hull^I_1, \dots,\Hull^I_{k-1},\Hull^i_k \equiv s\rangle$ be the sequence of vertexes of its convex hull, in counter-clockwise order.

Suppose that the convex hull $\ConvexHull{I}$ does not intersect the path \PathFromTo{s}{e}, except for the endpoints $\point_s$ and $\point_e$.

Then any non self-crossing tour containing 
the path $\PathFromTo{s}{e}$ reaches the vertexes in $\HullSetI$ in the order $\Hull^I_0, \dots \Hull^I_k$.
\end{theorem}

\begin{proof}
By contradiction, suppose that a non self-crossing circuit $\Circuit^* \supseteq \PathFromTo{s}{e}$ reaches the vertexes of $\Hull^I$ 
in an order different from their sequence order.
W.l.o.g., suppose that after vertex $\Hull^I_j$, the next vertex of $\Hull^I$ reached by the tour $\Circuit^*$ is $\Hull^I_{j+2}$;
i.e., vertex $\Hull^I_{j+1}$ is not reached in the order of $\Hull^I$.
Let $\PathFromTo{\Hull^I_j}{\Hull^I_{j+2}} \subset \Circuit^*$ be the path connecting $\Hull^I_j$
and $\Hull^I_{j+2}$.

Consider the polygon $R$ (dotted, in Fig.~\ref{fig:HullRicorsiva})
delimited by:
\begin{enumerate}[label=(\roman*), leftmargin=*]
\item the path $\PathFromTo{s}{e}$;
\item the perimeter of the convex hull $\HullSetI$ from $\point_e$
      to $\Hull^I_j$;
\item the path $\PathFromTo{\Hull^I_j}{\Hull^I_{j+2}}$;
\item the perimeter of the convex hull $\HullSetI$ from
      $\Hull^I_{j+2}$ to $\point_s$.
\end{enumerate}

Clearly, the point $\Hull^I_{j+1}$ lies in the interior of $R$.
In order to reach it without self-crossings, $\Circuit^*$ cannot pass through $\PathFromTo{s}{e}$
nor $\PathFromTo{\Hull^I_j}{\Hull^I_{j+2}}$.
On the other hand, $\Circuit^*$ cannot cross $\HullSetI$ between $\point_e$ and $\Hull^I_j$
(and from $\Hull^I_{j+2}$ to $\point_s$) because, by definition of convex hull, there are no vertexes to be visited that are interior to polygon $Q$
and that do not belong to $\HullSetI$.
So, $\Circuit^*$ cannot reach $\Hull^I_{j+1}$. \end{proof}

\begin{figure}[htbp]
\centering
\includegraphics[width=.9\textwidth]{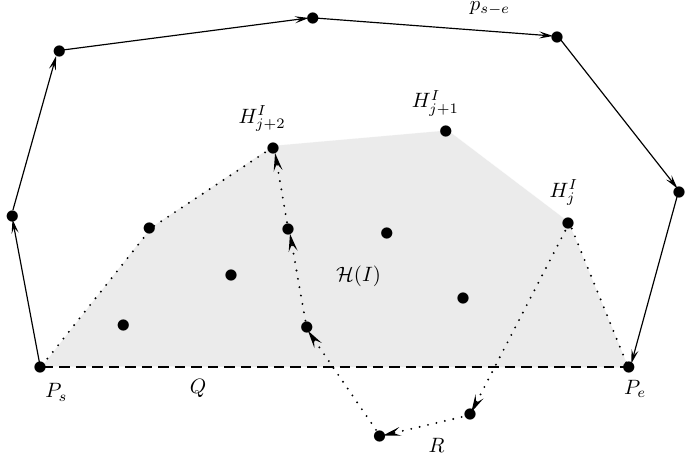}
\caption{\label{fig:HullRicorsiva} From Theorem~\ref{thm:HullRicorsiva}: the path $\PathFromTo{s}{e}$ is the current assignment. Polygon $Q$ is delimited by $\PathFromTo{s}{e}$
and the (dashed) segment $\segment{\point_e}{\point_s}$. $I$ contains the points strictly inside $Q$ plus $\point_s$ and $\point_e$; $\ConvexHull{I}$ (grey in the picture) is its convex hull. Polygon $R$ is delimited by $\PathFromTo{s}{e}$ and the dotted segments.}
\end{figure}

If we find an {\em internal} hull $\ConvexHull{I}$ (see Figure~\ref{fig:HullRicorsiva}), we can then apply the previous propagators 
(Algorithms~\ref{alg:clockwise_angle_propagator} and \ref{alg:hull_path_propagator})
also to $\HullSetI$, with the obvious care that if $\PathFromTo{s}{e}$ reaches points in clockwise order,
then $\HullSetI$ will be reached in counter-clockwise order (and vice versa). We maintain all partial paths during the search, and we compute a convex hull for each of these paths. In this way, we are orthogonal to heuristics (we can use any search heuristic without invalidating our propagation).

In the implementation, we applied the pruning on internal hulls
only when the polygon $Q$ is convex.
Checking the convexity amounts to check that each turn in the path is on the same side (a right turn, on a clockwise path),
it makes it easier to find when a point is inside the polygon and also allows us to avoid checking that $\ConvexHull{I}$ does not intersect $\PathFromTo{s}{e}$.

The time complexity of our implementation of the extended convex hull is $O(n^2)$ to compute the points inside the polygon, then we use Andrew's monotone chain ($O(\Npoints \log \Npoints)$)~\citep{AndrewMonotoneChain} to find the hull. We currently recompute $\HullSetI$ from scratch after each decision.

\section{Geometric reasoning for the Euclidean Generalized TSP} 
\label{sec:egtsp_hull}

The propagation mechanisms developed in Section~\ref{sec:nocrossing} are not limited to the Euclidean TSP, but can also be extended to other routing problems and variants of \ac{TSP} that use the Euclidean distance as a metric. 
As an example, consider the \acf{GTSP};
in the \ac{GTSP}, also known as \emph{set TSP}, the set $V$ of nodes of the graph is partitioned into $C_1, \dots, C_m$ subsets where each subset of nodes $C_i$ is called {\em cluster} and $C_j \cap C_k = \emptyset$ for all $j, k \in \lbrace 1, \dots, m\rbrace$. 
We use $\ClusterOp{\point_i}$ to denote the cluster associated with point ${\point_i}$. 
The objective is to compute the minimum cost cycle that visits each cluster at least once. 
The \ac{GTSP} has a wide number of applications~\citep{NoonBean1991, Laporte1996} including sequencing of computer files, routing of welfare customers through governmental agencies, design of ring networks, flexible manufacturing scheduling, airport selection and routing for courier planes,  and postal routing.
The \acf{EGTSP}, similarly to the \ac{ETSP}, uses Euclidean distance as its metric and the coordinates of the points on the plane are known. 

As in the \ac{ETSP}, we employed the successor representation to model the \ac{EGTSP}:
the constraint model (Equation~\ref{eq:set-constraintmodel}) involves \Npoints\ variables $\Next_i$. 
Since in the \ac{EGTSP} not all nodes are visited, we take the convention proposed in~\cite{explainingCircuit} that unvisited nodes have themselves as successors (i.e. $\Next_i = i$);
clearly this requires a modified version of the \circuit\ constraint~\cite{explainingCircuit} that accepts sub-tours of length 1 (Equation~\ref{eq:set-circuit}).
The model includes the {\tt alldifferent} constraint (Equation~\ref{eq:set-alldiff}) to avoid two nodes from having the same successor, and a constraint imposing
that at least one node is selected for each cluster (Equation~\ref{eq:set-cluster}).

\begin{subequations}
\begin{align}
{\tt alldifferent}(\Next) &\label{eq:set-alldiff} \\
{\tt set\_circuit}(\Next) & \label{eq:set-circuit} \\
\sum_{i \in C_k} \vert\Next_i - i\vert \geq 1 & \;\;\; \forall k \in \lbrace 1, \dots, m\rbrace \label{eq:set-cluster}
\end{align}
\label{eq:set-constraintmodel}
\end{subequations}

Having introduced the constraint model, we now investigate how the
geometric propagation developed for the \ac{ETSP} can be adapted to
the \ac{EGTSP}.

In \ac{EGTSP} we cannot directly apply the pruning introduced in Section~\ref{sec:convex} because the vertexes on the boundary of the convex hull may not be visited. Obviously in case such points are visited then they must be visited in the correct order, as already shown for the \ac{ETSP}. In the following we will show how, under certain conditions, such pruning can also be extended to other ``internal" points.

To simplify the following discussion, Definition~\ref{def:neighbours} introduces the concept of {\em set of neighbours} of a point on the boundary of the convex hull.

\begin{definition}
\label{def:neighbours}
Given $H_i \in \HullSet{\PointSet}$ a point $\point_j \in \ClusterSet{H_i}$ is called a {\em neighbour} of $\Hull_i$ if and only if the segment $\segment{\Hull_i}{\point_j}$ does not intersect $\ConvexHull{\PointSet \setminus \ClusterSet{H_i}}$.
The set $\Vicini{\Hull_i}$ is the set of all points $\point_j \in \ClusterSet{H_i}$ that are neighbours of $\Hull_i$ (by definition $\Hull_i \in \Vicini{\Hull_i}$).
\end{definition}

In the example in Figure~\ref{fig:neighbours}, points have a color representing the cluster they belong to.
Among the points in the green cluster, only $\point_k$ is not a neighbor of $\Hull_i$,
because the segment $\segment{\Hull_i}{\point_k}$ intersects the convex hull $\ConvexHull{\PointSet \setminus \ClusterSet{H_i}}$ of the non-green points.

Now we can extend the reasoning of the convex hull to the \ac{GTSP}.

\begin{figure}[tbp]
		\centering
\includegraphics[width=.65\textwidth]{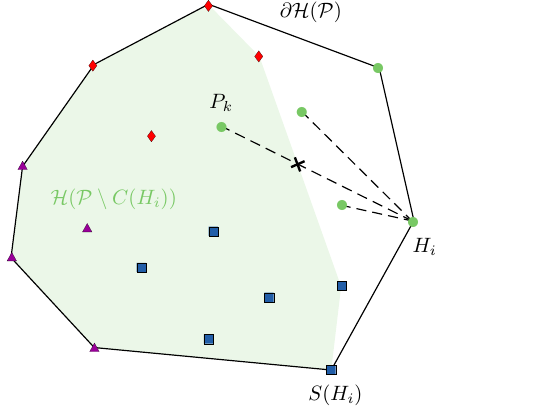}
\caption{From Definition~\ref{def:neighbours} $\point_k$ is not a neighbor of $\Hull_i$ because the segment $\segment{\Hull_i}{\point_k}$ intersects the convex hull $\ConvexHull{\PointSet \setminus \ClusterSet{H_i}}$.}
	\label{fig:neighbours}
\end{figure}



\begin{theorem}
\label{thm:hull_neighbours}

Let $\point_g \equiv \Hull_i$ be a point in
$\HullSet{\PointSet}$, and let
$\point_b \equiv \HSucc{\Hull_i}{\PointSet}$
be its successor in the sequence $\HullSet{\PointSet}$.
Let $C_g$ and $C_b$ be the clusters containing $\point_g$ and
$\point_b$, respectively, with $C_g \neq C_b$.

Let
$
\Phi_{\point_g}
= \ConvexHull{\Vicini{\point_g} \cup \Vicini{\point_b}}
$
be the convex hull of the neighbours of $\point_g$ and $\point_b$,
and let
$
\Phi_{\point_g}^{'}
= \ConvexHull{
    \PointSet \setminus
    \left(
        \ClusterSet{\point_g} \cup \ClusterSet{\point_b}
    \right)
}
$
be the convex hull of all points belonging to clusters other than
$C_g$ and $C_b$.

If
$\Phi_{\point_g} \cap \Phi_{\point_g}^{'} = \emptyset$,
then no point in $\Vicini{\point_b}$ can have a point in
$\Vicini{\point_g}$ as its successor in any optimal \ac{EGTSP}.
\end{theorem}

\begin{proof}
Let points $\point_{g^{'}}$ and $\point_{b^{'}}$ be neighbour of point $\point_g$ and neighbour of point $\point_b$ respectively, and let $c^{*}$ be an optimal \ac{EGTSP} that includes both the point $\point_{g^{'}}$ and the point $\point_{b^{'}}$.
Firstly, we want to prove that both points $\point_{g^{'}}$ and $\point_{b^{'}}$ are in $\HullSet{c^{*}}$.
By definition of neighbour of a point on the boundary of the convex hull, no point in the segment $\segment{\point_g}{\point_{g'}}$ is in $\ConvexHull{\PointSet \setminus \ClusterSet{\point_g}}$, so every point $\point_j$ of the segment is on the border $\HullSet{\left(\PointSet \setminus \ClusterSet{\point_g}\right) \cup \lbrace \point_j \rbrace}$.
In an optimal \ac{EGTSP} there is only one point for each cluster so if we have selected $\point_{g^{'}}$ there cannot be any other point belonging to cluster $C_g$, 
it follows that $c^{*} \subset \ConvexHull{\left(\PointSet \setminus \ClusterSet{\point_g}\right) \cup \lbrace \point_{g^{'}} \rbrace}$; 
since $\point_{g^{'}} \in \HullSet{\left( \PointSet \setminus \ClusterSet{\point_g}\right) \cup \lbrace \point_{g^{'}} \rbrace}$, also
$\point_{g^{'}} \in \HullSet{c^{*}}$.
A similar argument holds for the point $\point_{b^{'}}$, so we proved that they are both points in $\HullSet{c^{*}}$.

Suppose now, by contradiction, that $\point_{b^{'}}$ has $\point_{g^{'}}$ as its successor in $c^{*}$ (see Figure~\ref{fig:neighbour-proof-1}).
Since $\point_{g'}$ is in the border $\HullSet{c^{*}}$,
 the angle formed by the segment $\segment{\point_{b'}}{\point_{g'}}$ and the segment from $\point_{g'}$ to its successor must correspond to a right turn (it must be between $-\pi$ and $0$). 
It follows that the successor of $\point_{g'}$ in $c^*$, denoted by $P_k$,
must be to the right of the straight line $\StraightLine{\point_{b'}}{\point_{g'}}$ for $c^{*}$ to be optimal.

According to the position of the point $\point_k$, it is necessary to distinguish 3 different cases. 

In the first case the point $\point_k$ lies in the interior of polygon $\point_b\point_b^{'}\point_g\point_g^{'}$ and therefore $\point_k \in \Phi_{\point_g}$ but simultaneously $\point_k \in \Phi_{\point_g}^{'}$ because $\ClusterSet{\point_k} \neq \ClusterSet{\point_g}$ and $\ClusterSet{\point_k} \neq \ClusterSet{\point_b}$ contradicting the condition $\Phi_{\point_g} \cap \Phi_{\point_g}^{'} = \emptyset$ (see Figure~\ref{fig:neighbour-proof-2}).

In the second case both segments $\segment{\point_{k}}{\point_{g'}}$ and $\segment{\point_{k}}{\point_{b'}}$ intersect the segment $\segment{\point_{g}}{\point_{b}}$ so it is not true that $\point_b$ is the successor of $\point_g$ in $\HullSet{\PointSet}$ (see Figure~\ref{fig:neighbour-proof-3}).

If the first two cases do not occur then either $\segment{\point_{k}}{\point_{b'}}$ intersects $\segment{\point_{g}}{\point_{g'}}$ or $\segment{\point_{k}}{\point_{g'}}$ intersects $\segment{\point_{b}}{\point_{b'}}$. In the case where $\segment{\point_{k}}{\point_{b'}}$ intersects $\segment{\point_{g}}{\point_{g'}}$ it means that $\ConvexHull{\PointSet \setminus \ClusterSet{\point_g}}$ intersects $\segment{\point_{g}}{\point_{g'}}$  so $\point_{g'}$ by definition cannot be a neighbour of $\point_g$ (see Figure~\ref{fig:neighbour-proof-4}). For the same reason $\point_{b'}$ is not neighbour of $\point_b$ if $\segment{\point_{k}}{\point_{g'}}$ intersects $\segment{\point_{b}}{\point_{b'}}$.

\end{proof}

\begin{figure}[tbp]
    \centering
    \begin{subfigure}{.5\textwidth}
            \centering
			\includegraphics[width=\textwidth, height=\textwidth]{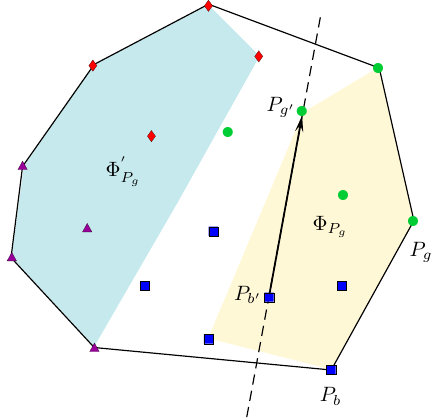}
			 \caption{}\label{fig:neighbour-proof-1}
    \end{subfigure}%
    \hfill 
    \begin{subfigure}{.5\textwidth}
            \centering
			\includegraphics[width=\textwidth, height=\textwidth]{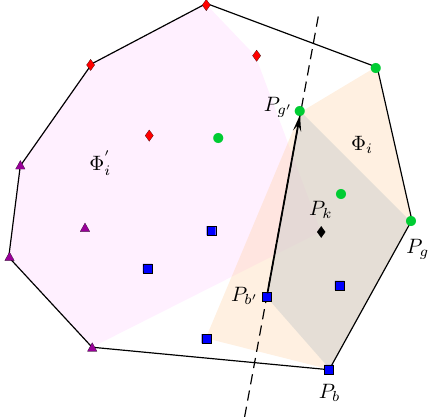}
			 \caption{}\label{fig:neighbour-proof-2}
		    \end{subfigure}

		 \vskip\baselineskip
	     
            \begin{subfigure}{.5\textwidth}
            \centering
			\includegraphics[width=\textwidth, height=\textwidth]{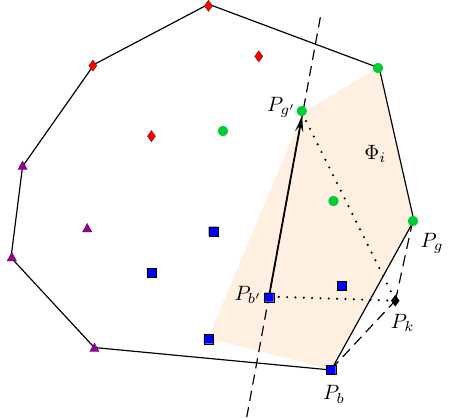}
			 \caption{}
			 \label{fig:neighbour-proof-3}
    \end{subfigure}%
    \hfill 
    \begin{subfigure}{.5\textwidth}
            \centering
			\includegraphics[width=\textwidth, height=\textwidth]{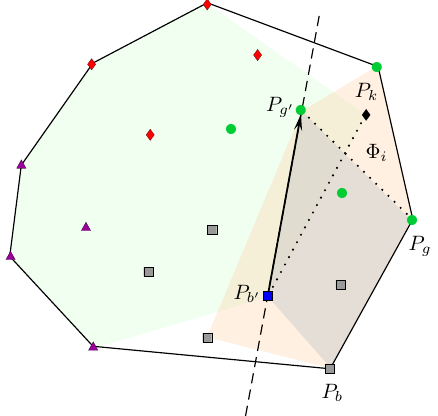}
			 \caption{}\label{fig:neighbour-proof-4}
		    \end{subfigure}

		 \caption{From Theorem~\ref{thm:hull_neighbours} $\point_{b^{'}}$ cannot have $\point_{g^{'}}$ as successor in any optimal \ac{EGTSP} if $\Phi_{\point_g} \cap \Phi_{\point_g}^{'} = \emptyset$.}
		 \label{fig:neighbour-proof}
\end{figure}

The implementation of a propagator exploiting Theorem~\ref{thm:hull_neighbours} is given in Algorithm~\ref{alg:hullset_propagator}.

In the proof of the Theorem~\ref{thm:hull_neighbours} we also proved that neighbours of a point in $\HullSetP$, if selected, will be on the boundary of the convex hull of the optimal solution so it is possible to impose the clockwise constraint (see Section~\ref{sec:convex}) on them as well. 
In the experimental evaluation, we applied only the reasoning on the angle formed by the incoming and the outgoing arcs in hull vertexes, using the easiest implementation, which consists of waiting until variables become instantiated with a value.

\begin{algorithm}[t]
\small
\caption{hull\_set\_propagator \label{alg:hullset_propagator}}

\begin{algorithmic}[1]

\Function{find\_neighbours}{$\HullSetP,\Next$}
\ForEach {$\Hull_i \in \HullSetP$}
	\State $\Vicini{\Hull_i} \leftarrow \emptyset$
\ForEach {$\point_j \in \ClusterSet{H_i}$}
		\If {$\segment{\Hull_i}{\point_j}$ does \textbf{not} intersect $\ConvexHull{\PointSet \setminus \ClusterSet{H_i}}$}
			\State $\Vicini{\Hull_i} \leftarrow \Vicini{\Hull_i} \cup \point_j$
		\EndIf
	\EndFor
\EndFor
\ForEach {$\Hull_i \in \HullSetP$}
\State \Call {hull\_set\_propagator}{$\HullSetP, \Next, \Vicini{\Hull_i}, \Vicini{\HSucc{\Hull_i}{\PointSet}}$}
\EndFor
\EndFunction

\Function{hull\_set\_propagator}{$\HullSetP, \Next, \Vicini{\Hull_i}, \Vicini{\HSucc{\Hull_i}{\PointSet}}$}
	\State ${\Phi}_{\Hull_i} \leftarrow \ConvexHull{\Vicini{H_i} \cup \Vicini{\HSucc{\Hull_i}{\PointSet}}}$
	\State ${\Phi}^{'}_{\Hull_i} \leftarrow \ConvexHull{\PointSet \setminus (\ClusterSet{\Hull_i} \cup \ClusterSet{\HSucc{\Hull_i}{\PointSet}})}$
	\If {${\Phi}_{\Hull_i} \cap {\Phi}^{'}_{\Hull_i} = \emptyset$}
\ForEach {$\point_j \in \Vicini{\HSucc{\Hull_i}{\PointSet}}$}
					\State remove $\Vicini{\Hull_{i}}$ from $\Dom{\Next_{\point_j}}$
				\EndFor
\Else
		\State suspend waiting for $\HullSet{{\Phi}^{'}_{\Hull_i}}$ to be modified
	\EndIf
\EndFunction

\end{algorithmic}
\end{algorithm}

 \section{Experimental Evaluation}
\sloppy
\label{sec:experiments}
To assess the effectiveness of the proposed algorithms, we devised 
experiments for both Euclidean \acp{TSP} and Euclidean Generalized \acp{TSP}.
All algorithms are implemented in the \ECLiPSe\ CLP language \cite{ECLiPSe}.
All tests were run on \ECLiPSe\  v. 7.0, build \#48, 
with a time limit of 1800s on Intel\textsuperscript{\textregistered} Xeon\textsuperscript{\textregistered} E5-2630 v3 CPUs running at 2.4GHz,
using only one core and with 2GB of reserved memory.
The \ECLiPSe\ Constraint Programming System is distributed as open-source software\footnote{\url{https://eclipseclp.org/}}.
The implementation uses the suspension and event-based wake-up
facilities provided by \ECLiPSe. These facilities are used to
react to domain modifications or variable instantiation and to
reschedule the propagators only when relevant changes occur.
However, the filtering rules are not specific to \ECLiPSe.
They can be ported to other \ac{CLP} or \ac{CP} systems that support
user-defined propagators, inspection and modification of finite
domains, and event-based activation on domain changes.

\subsection{Euclidean TSP}

Euclidean \ac{TSP} constraint model is based on the successor representation described in the preliminaries (with the \circuit\ constraint and \alldifferent\ \cite{PugetAlldifferent} for improved pruning) with both $\Next$ and \Prev\ variables, that are linked through the {\tt inverse} constraint. In this section we compare different constraint models extending the one presented in the preliminaries by adding additional constraints.

In the following we denote by \noi\ the constraint model that adds the \nocrossing\ constraint,
the removal of aligned points according to Proposition~\ref{thm:three_points_aligned}, and the {\tt clockwise} constraint that implements the propagation described in Section~\ref{sec:convex}.

In order to show that the pruning we provide is not subsumed by that of state-of-the-art techniques, we implemented in \ECLiPSe, in the successor representation, also the  
Held and Karp bound with pruning based on Reduced and Marginal Costs, as proposed by \citep{vanHoeveEtAl_ImprovedWeightedCircuit}
(shown with \loro\ in the following). The purpose of this comparison is to isolate the incremental effect of the proposed geometric propagation within a common \ac{CLP} implementation. 

As 
search strategies, 
we use the \maxregret~\citep{CaseauLaburthe}
and the state-of-the-art \lcfirst~\citep{SalesmanAndTree}, based on {\em Last Conflict}~\citep{LastConflict}. 
As~\citep{vanHoeveEtAl_ImprovedWeightedCircuit}, we also experimented injecting the upper bound given by the \ac{LKH} (v. 2.0.7) algorithm. 

We experimented both with structured and with random (unstructured) instances.

\subsubsection{Structured instances}
\label{subsec:structured}

Structured instances model the position of real locations or sets of points from real-world problems.
We consider instances taken from the TSPLIB~\citep{tsplib}, the Concorde website\footnote{\url{http://www.math.uwaterloo.ca/tsp/world/countries.html}} and the CITIES dataset\footnote{\url{https://people.sc.fsu.edu/~jburkardt/datasets/cities/cities.html}} up to 100 nodes.
These sources provide various types of instances, amongst which Euclidean ones, represented as sets of points in the plane, and geographical ones, represented as sets of points (with latitude and longitude) on the surface of the Earth.
In order to have a more complete set of instances, we decided to consider also geographical instances in which the cities to be visited lie on a limited part of the geoid, so that the geographical distance can be approximated with the Euclidean distance.

Table~\ref{tab:experiments_tsplib} reports the results for the
19 structured instances for which at least one configuration reached
the optimal solution within the 1800-second time limit. Six of the
25 tested instances were omitted because all compared
configurations reached the time limit. Each instance listed in the table is represented by its name, followed by the solving time (in seconds) and the number of search nodes required to find the optimal solution and prove its optimality.
Aside from simple instances, the constraint models that contain geometric filtering are the ones that optimally solve the instances in the shortest time. These results show the positive interaction between the geometric filtering and that carried out by \loro\ alone. By analysing the instances that were not solved to optimality, we found that our additional pruning is more effective during the proof of optimality, rather than on finding good initial solutions.

\newcommand{\timeout}{T/O}
\newcommand{\noditimeout}[1]{\multicolumn{1}{@{}c@{}}{-}}
\newcommand{\noditimeoutmargin}[1]{\multicolumn{1}{@{}c@{}}{-}}
\begin{table}[tbp]
\centering
\caption{\label{tab:experiments_tsplib}
Comparing filtering algorithms on structured instances with time limit 1800s. For each instance we report total solving time and number of explored nodes to reach the optimal solution and prove its optimality.}
{\tablefont\begin{tabular}{@{\extracolsep{\fill}}lrrrrrrrr}
	\topline
	& \multicolumn{4}{c}{\lcfirst} & \multicolumn{4}{c}{\maxregret}\\
	& \multicolumn{2}{c}{\loro} & \multicolumn{2}{c}{\noipiuloro} & \multicolumn{2}{c}{\loro} & \multicolumn{2}{c}{\noipiuloro}\\[4pt]
	instances & \shortstack{CPU\\time} & nodes & \shortstack{CPU\\time} & nodes & \shortstack{CPU\\time} & nodes & \shortstack{CPU\\time} & node
	\midline
	uk12                          & \textbf{0.04}            & 12                                                                         & 0.05                     & 12                                                                         & \textbf{0.04}            & \textbf{1}                                                                 & \textbf{0.04}            & 1                                                                          \\
	burma14                       & 0.06                     & 14                                                                         & 0.07                     & 14                                                                         & \textbf{0.05}            & \textbf{1}                                                                 & 0.07                     & \textbf{1}                                                                 \\
	ulysses16                     & \textbf{0.07}            & 16                                                                         & 0.09                     & 16                                                                         & \textbf{0.07}            & \textbf{1}                                                                 & 0.08                     & 1                                                                          \\
	ulysses22                     & \textbf{0.12}            & 22                                                                         & 0.15                     & 22                                                                         & \textbf{0.12}            & \textbf{1}                                                                 & 0.16                     & \textbf{1}                                                                 \\
	wg22                          & 0.14                     & 22                                                                         & 0.18                     & 22                                                                         & \textbf{0.13}            & \textbf{1}                                                                 & 0.17                     & \textbf{1}                                                                 \\
	bayg29                        & 0.45                     & 30                                                                         & 0.44                     & 35                                                                         & \textbf{0.41}            & 3                                                                          & 0.43                     & \textbf{2}                                                                 \\
	wi29                          & \textbf{0.23}            & 29                                                                         & 0.32                     & 29                                                                         & \textbf{0.23}            & \textbf{1}                                                                 & 0.32                     & \textbf{1}                                                                 \\
	dj38                          & \textbf{0.42}            & 38                                                                         & 0.65                     & 38                                                                         & \textbf{0.42}            & \textbf{1}                                                                 & 0.65                     & \textbf{1}                                                                 \\
	dantzig42                     & 2.16                     & 48                                                                         & \textbf{1.60}            & 43                                                                         & 4.83                     & 28                                                                         & 2.11                     & \textbf{3}                                                                 \\
	att48                         & 3.21                     & 66                                                                         & \textbf{2.38}            & 55                                                                         & 8.02                     & 36                                                                         & 3.10                     & \textbf{10}                                                                \\
	uscap50                       & \textbf{0.64}            & 50                                                                         & 1.22                     & 50                                                                         & \textbf{0.64}            & \textbf{1}                                                                 & 1.22                     & \textbf{1}                                                                 \\
	eil51                         & T/O                      & -                                                                          & 523.92                   & 1384                                                                       & 313.53                   & 916                                                                        & \textbf{84.43}           & \textbf{195}                                                               \\
	berlin52                      & 0.86                     & 52                                                                         & 1.67                     & 52                                                                         & \textbf{0.84}            & \textbf{1}                                                                 & 1.52                     & \textbf{1}                                                                 \\
	kn57                          & 8.53                     & 180                                                                        & 5.67                     & 120                                                                        & 8.24                     & 38                                                                         & \textbf{5.14}            & \textbf{18}                                                                \\
	wg59                          & \textbf{1.13}            & 59                                                                         & 1.92                     & 59                                                                         & 1.35                     & \textbf{1}                                                                 & 1.8                      & \textbf{1}                                                                 \\
	st70                          & 1411.63                  & 3934                                                                       & \textbf{327.79}          & \textbf{985}                                                               & T/O                      & -                                                                          & T/O                      & -                                                                          \\
	eil76                         & 72.56                    & 227                                                                        & 33.20                    & 156                                                                        & 160.09                   & 230                                                                        & \textbf{31.8}            & \textbf{26}                                                                \\
	rat99                         & T/O                      & -                                                                          & T/O                      & -                                                                          & T/O                      & -                                                                          & \textbf{436.97}          & \textbf{574}                                                               \\
	rd100                         & 25.81                    & 117                                                                        & \textbf{21.79}           & 107                                                                        & 52.99                    & 80                                                                         & 21.94                    & \textbf{22}                                                                
	\botline
	\end{tabular}}
	\end{table}

\subsubsection{Random instances}
\label{subsec:random}
Random or unstructured instances consist of points with random coordinates on the plane. 
In the literature~\citep{johnson2007experimental, AdditiveBoundingTSP}, typically two types of instances  are considered: {\em uniform} and {\em clustered}.
Instances of the uniform class are obtained by placing points in the plane according to a uniform distribution,
while clustered instances consist of clusters of points, whose centres are uniformly distributed in the plane.
Random instances have been generated through the R-package {\tt netgen}~\citep{bossek2015netgen} and its functions: {\tt generateRandomNetwork}, {\tt generateClusteredNetwork}.
The function {\tt generateRandomNetwork} generates a random graph by placing points uniformly distributed in a $10^6 \times 10^6$ square.
The function {\tt generateClusteredNetwork} initially generates cluster centers by \ac{LHS} to ensure that the clusters are placed separately from each other. Then it distributes points to the clusters according to a normal distribution having as mean vector the centers of the clusters and as variance the distance to the center of the nearest neighbor cluster.

We randomly generated instances from 20 to 100 nodes in steps of 2, in both classes. 
For each size in the uniform class we generated 32 instances, for a total of 1312 instances. 
For the clustered class we tested $n.cluster = 3, 5, 8, 11$ as number of clusters;
for each size and number of clusters we generated 16 instances, for a total of 2624 instances.
For the uniform class, 276 of the 1312 generated instances
(21.0\%) were excluded from the plots and aggregate statistics
because all compared configurations reached the 1800-second time
limit. For the clustered class, this occurred for 960 of the 2624
generated instances (36.6\%).

In Figure~\ref{graph:uniform_lcfirst}, we present log-log scatter plots and boxplots comparing the \noipiuloro\ model, which includes our filtering algorithm, to the \loro\ model when solving uniform instances with the \lcfirst\ search strategy. Figures \ref{graph:scatter_uniform_lcfirst_time} and \ref{graph:boxplot_uniform_lcfirst_time} compare the solving time, while Figures \ref{graph:scatter_uniform_lcfirst_nodes} and \ref{graph:boxplot_uniform_lcfirst_nodes} compare the search nodes. 
Figure~\ref{graph:scatter_uniform_lcfirst_time} shows that the solving time is significantly reduced when the \noipiuloro\ model is used compared to \loro. The reduction in solving time is evident across all node ranges, with the most substantial improvements observed for larger instances. On average, there is a 70\% reduction in solving time when our filtering is used. Figure~\ref{graph:boxplot_uniform_lcfirst_time} provides a boxplot representation of the solving time across different instances size;
note that, even on a logarithmic scale, the gap between the median solving times increases with the problem size, reaching an order of magnitude for instances with 80 or more nodes, despite the saturating effect of the timeout encountered by the competitor method for larger instance sizes.
Figure \ref{graph:scatter_uniform_lcfirst_nodes} and \ref{graph:boxplot_uniform_lcfirst_nodes} show the number of search nodes required to reach a solution. Similar to the solving time, the \noipiuloro\ model consistently requires fewer search nodes than the \loro\ model, with an average reduction of 59\%.
Notably, this reduction holds even though, in timeout cases, the node count includes only the nodes explored up to the timeout threshold, potentially underestimating the effort required by the \loro\ model and thus disadvantaging the \noipiuloro\ model.

\begin{figure}[htbp]
    \centering
    \begin{subfigure}{.475\textwidth}
	     \includegraphics[width=\textwidth, height=\textwidth]{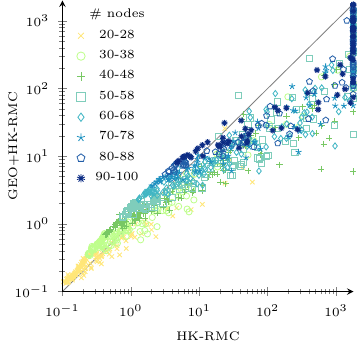}
		 \caption{CPU Time}
		 \label{graph:scatter_uniform_lcfirst_time}
    \end{subfigure}%
    \hfill
    \begin{subfigure}{.475\textwidth}    
	    \includegraphics[width=\textwidth, height=\textwidth]{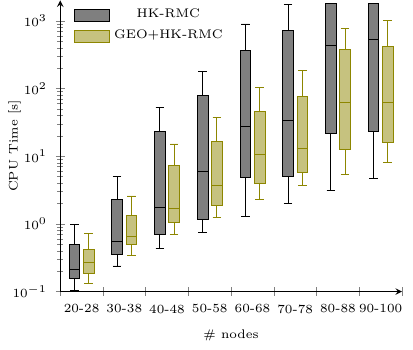}
		 \caption{CPU Time}  
		 \label{graph:boxplot_uniform_lcfirst_time}
	\end{subfigure}
		 
 	\vskip\baselineskip
     \begin{subfigure}{.475\textwidth}
     	\includegraphics[width=\textwidth, height=\textwidth]{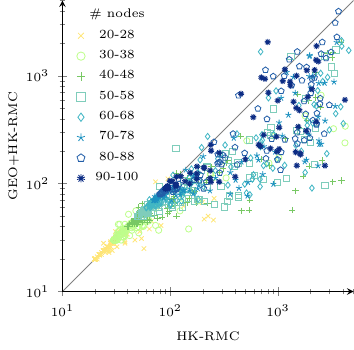}
	 	\caption{Search Nodes}
	 	\label{graph:scatter_uniform_lcfirst_nodes}
    \end{subfigure}%
    \hfill
    \begin{subfigure}{.475\textwidth}    
    	\includegraphics[width=\textwidth, height=\textwidth]{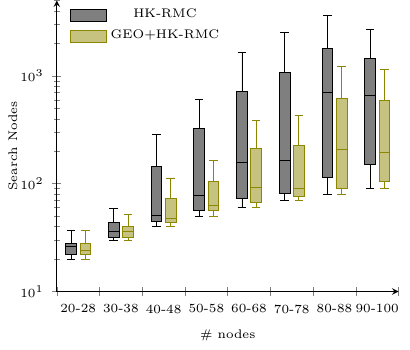}
	 	\caption{Search Nodes}  
	 	\label{graph:boxplot_uniform_lcfirst_nodes}
	 \end{subfigure}

		 \caption{Comparing solving time \ref{graph:scatter_uniform_lcfirst_time}, \ref{graph:boxplot_uniform_lcfirst_time} and search nodes \ref{graph:scatter_uniform_lcfirst_nodes}, \ref{graph:boxplot_uniform_lcfirst_nodes} of \loro\ and \noipiuloro\ models on randomly-generated Euclidean \ac{TSP} instances belonging to the \textbf{uniform} class when \lcfirst\ search strategy is used. The solving time limit was set to 1,800 seconds. Each data point in these figures corresponds to an instance. We removed instances where both models ran into timeout.}
		 \label{graph:uniform_lcfirst}
\end{figure}

In Figure~\ref{graph:uniform_maxregret}, we present similar log-log scatter plots and boxplots comparing the \noipiuloro\ and \loro\ constraint models when using the \maxregret\ search strategy.
As with \lcfirst, the \noipiuloro\ model achieves substantial improvements in both solving time and number of search nodes. Specifically, it reduces solving time by an average of 70\% and the number of search nodes by 71\%, with nearly all instances lying below the diagonal in the scatter plots. Notably, several small instances are solved by the \noipiuloro\ model using just one search node.
Similar issues with timeouts apply here as well. When an instance times out, we only count the nodes explored up to the time limit. This likely underestimates how many nodes the \loro\ model would actually need, making the comparison unfairly harder for our model.

\begin{figure}[htbp]
    \centering
    \begin{subfigure}{.475\textwidth}
	     \includegraphics[width=\textwidth, height=\textwidth]{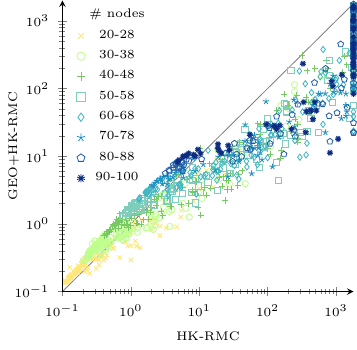}
		 \caption{CPU Time}
		 \label{graph:scatter_uniform_maxregret_time}
    \end{subfigure}%
    \hfill
    \begin{subfigure}{.475\textwidth}    
	    \includegraphics[width=\textwidth, height=\textwidth]{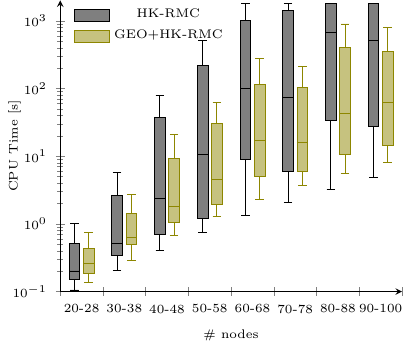}
		 \caption{CPU Time}  
		 \label{graph:boxplot_uniform_maxregret_time}
		 \end{subfigure}
		 
		 \vskip\baselineskip
	     \begin{subfigure}{.475\textwidth}
            \includegraphics[width=\textwidth, height=\textwidth]{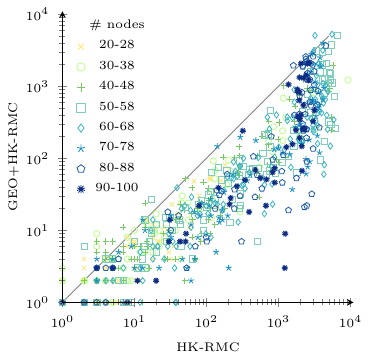}
			 \caption{Search Nodes}
			 \label{graph:scatter_uniform_maxregret_nodes}
		    \end{subfigure}%
		    \hfill
		    \begin{subfigure}{.475\textwidth}    
			\includegraphics[width=\textwidth, height=\textwidth]{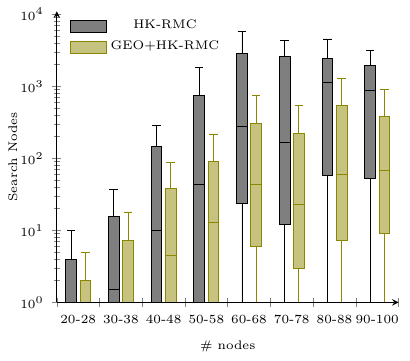}
	 \caption{Search Nodes}  
	 \label{graph:boxplot_uniform_maxregret_nodes}
	 \end{subfigure}

		 \caption{Comparing solving time \ref{graph:scatter_uniform_maxregret_time}, \ref{graph:boxplot_uniform_maxregret_time} and search nodes \ref{graph:scatter_uniform_maxregret_nodes}, \ref{graph:boxplot_uniform_maxregret_nodes} of \loro\ and \noipiuloro\ models on randomly-generated Euclidean \ac{TSP} instances belonging to the \textbf{uniform} class when \maxregret\ search strategy is used. The solving time limit was set to 1,800 seconds. Each data point in these figures corresponds to an instance. We removed instances where both models ran into timeout.}
		 \label{graph:uniform_maxregret}
\end{figure}

Table~\ref{tab:uniform_timeouts} presents the number of instances that reached the timeout limit of 1800 seconds during the experiments for various configurations of nodes and constraint models. The results indicate that for small instances (20 to 38 nodes), both algorithms performed similarly well with almost no instances reaching the timeout. As the dimension increased, the \noipiuloro\ algorithm consistently outperformed \loro, particularly with the \lcfirst\ search strategy. For example, in the 90-100 node range, \noipiuloro\ with \lcfirst\ had 119 timeouts compared to 144 of \loro.

\FloatBarrier
\begin{table}[htb]
	\centering
	\caption{Number of uniform instances that were not solved within the timeout during experiments grouped by size (shown as nodes range). The experiments compared the \loro\ and \noipiuloro\ constraint models, with the \maxregret\ and \lcfirst\ search strategies.}
	{\tablefont\begin{tabular}{@{\extracolsep{\fill}}ccccc}
	\topline
	         & \multicolumn{2}{c}{\maxregret}                             & \multicolumn{2}{c}{\lcfirst}                               \\
	\# nodes & \loro & \noipiuloro & \loro & \noipiuloro 
	\midline
	20-28                        & \textbf{0}                & \textbf{0}                    & \textbf{0}                & \textbf{0}                    \\
	30-38                        & 1                         & \textbf{0}                    & 5                         & 1                             \\
	40-48                        & 4                         & \textbf{1}                    & 12                        & 3                             \\
	50-58                        & 34                        & \textbf{13}                   & 35                        & 26                            \\
	60-68                        & 57                        & \textbf{29}                   & 51                        & 34                            \\
	70-78                        & 90                        & 68                            & 84                        & \textbf{65}                   \\
	80-88                        & 111                       & 78                            & 105                       & \textbf{68}                   \\
	90-100                       & 156                       & 131                           & 144                       & \textbf{119}                  \botline
	\end{tabular}}
	\label{tab:uniform_timeouts}
	\end{table}

Experiments on clustered instances, as shown in Figure~\ref{graph:scatter_clustered}, confirm that the use of geometric filtering significantly reduces the average solving time by 70\% when compared to the \loro\ model alone.

Figure~\ref{graph:scatter_clustered_lc_first_time} presents a log-log scatter plot comparing the solving time of the \noipiuloro\ model with the \loro\ model across different node ranges and cluster counts when \lcfirst\ search strategy is used. Each point represents an instance, and the results show a clear reduction in solving time for the \noipiuloro\ model.
Figure~\ref{graph:scatter_clustered_lc_first_nodes} shows a similar scatter plot but focuses on the number of search nodes. The \noipiuloro\ model consistently requires fewer search nodes to reach a solution, indicating more efficient search performance.

Figures \ref{graph:scatter_clustered_maxregret_time} and \ref{graph:scatter_clustered_maxregret_nodes} replicate the analysis using the \maxregret\ search strategy. Figure \ref{graph:scatter_clustered_maxregret_time} shows the solving time comparison, while Figure \ref{graph:scatter_clustered_maxregret_nodes} presents the number of search nodes. The patterns observed are similar to those with the \lcfirst\ strategy, further validating the robustness of the \noipiuloro\ model.

Specifically, the \noipiuloro\ model provided substantial reductions in the number of explored nodes: a 43\% reduction with the \lcfirst\ strategy and a 75\% reduction with the \maxregret\ search strategy.
From Figure~\ref{graph:scatter_clustered}, it can also be observed that the number of clusters does not significantly affect the solving performance of the two constraint models. This suggests that the improvements provided by the geometric filtering are robust across varying cluster configurations.

	\begin{figure}[htbp]
		\centering
		\begin{subfigure}{.475\textwidth}
		\centering 
			\includegraphics[width=\textwidth, height=\textwidth]{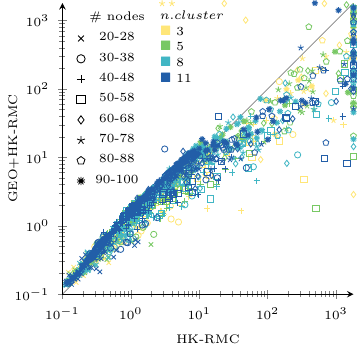}
	 \caption{CPU Time [s]}
		\label{graph:scatter_clustered_lc_first_time}
		\end{subfigure}%
		\hfill 
		\begin{subfigure}{.475\textwidth}
		\centering     
			\includegraphics[width=\textwidth, height=\textwidth]{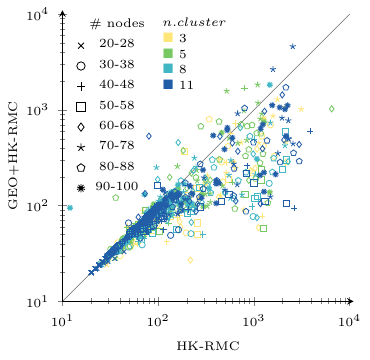}
			\caption{Search Nodes}
			\label{graph:scatter_clustered_lc_first_nodes}
		\end{subfigure}
	\vskip\baselineskip	
		
			\begin{subfigure}{.475\textwidth}
			\centering 
			\includegraphics[width=\textwidth, height=\textwidth]{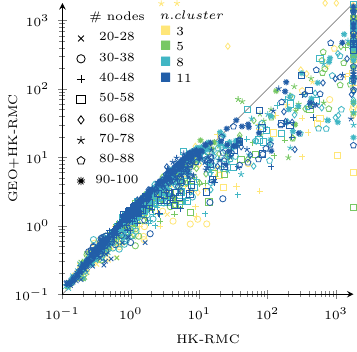}
	 \caption{CPU Time [s]}
		\label{graph:scatter_clustered_maxregret_time}
		\end{subfigure}%
		\hfill 
		\begin{subfigure}{.475\textwidth}   
			\centering 
			\includegraphics[width=\textwidth, height=\textwidth]{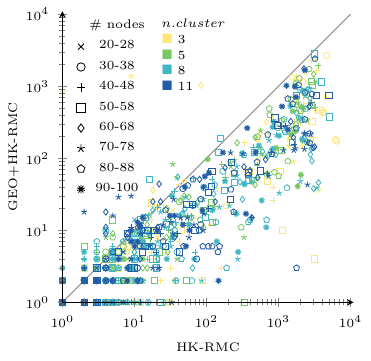}
			\caption{Search Nodes}
			\label{graph:scatter_clustered_maxregret_nodes}
		\end{subfigure}
		
			\caption{Comparing solving time and search nodes of \loro\ and \noipiuloro\ models on randomly-generated Euclidean \ac{TSP} instances belonging to the \textbf{clustered} class varying the number of clusters (shown by the different colors). 
			 Figures \ref{graph:scatter_clustered_lc_first_time} and \ref{graph:scatter_clustered_lc_first_nodes} show the results obtained using the \lcfirst\ search strategy while Figures \ref{graph:scatter_clustered_maxregret_time} and \ref{graph:scatter_clustered_maxregret_nodes} show the results obtained using the \maxregret\ search strategy.
			 The solving time limit was set to 1,800 seconds. We removed instances where both models ran into timeout.}
		\label{graph:scatter_clustered}
	\end{figure}

Table~\ref{tab:clustered_timeouts} presents the number of clustered instances that reached the timeout limit of 1800 seconds during the experiments, categorized by different configurations of clusters, dimensions, and constraint model.
The trend concerning the instance size mirrors that observed in the uniform instances. For smaller instances (up to 68 nodes), the \maxregret\ strategy generally performs better. However, for larger instances, the \lcfirst\ strategy shows superior performance. As for the constraint models, the \noipiuloro\ model consistently outperforms the \loro\ model across different configurations.
When analyzing the number of clusters, a clear pattern emerges: instances with a higher number of clusters are generally easier to solve, as reflected by the reduced number of timeouts. This suggests that increasing the number of clusters simplifies the problem structure, making instances more tractable within the given time limit.

	\begin{table}[htb]
		\centering
		\caption{Number of clustered instances that reached the timeout limit of 1800 seconds during the experiments grouped by different number of clusters and size (shown as nodes range). The experiments compared the \loro\ and \noipiuloro\ algorithms, with the \maxregret\ and \lcfirst\ search strategies.}
		{\tablefont\begin{tabular}{@{\extracolsep{\fill}}cccccc}
		\topline
							&          & \multicolumn{2}{c}{\maxregret} & \multicolumn{2}{c}{\lcfirst} \\
		\# clusters         & \# nodes & \loro         & \noipiuloro     & \loro        & \noipiuloro    
		\midline
		\multirow{8}{*}{3}  & 20-28    & \textbf{0}    & \textbf{0}    & \textbf{0}   & \textbf{0}   \\
							& 30-38    & 9             & \textbf{4}    & 11           & 8            \\
							& 40-48    & 28            & \textbf{23}   & 30           & 28           \\
							& 50-58    & 55            & \textbf{50}   & 52           & 51           \\
							& 60-68    & 67            & \textbf{63}   & 64           & \textbf{63}  \\
							& 70-78    & 74            & 72            & 67           & \textbf{67}  \\
							& 80-88    & 78            & 76            & 76           & \textbf{75}  \\
							& 90-100   & 96            & 95            & 91           & \textbf{90}  \\ \hline
		\multirow{8}{*}{5}  & 20-28    & \textbf{0}    & \textbf{0}    & \textbf{0}   & \textbf{0}   \\
							& 30-38    & \textbf{0}    & \textbf{0}    & \textbf{0}   & \textbf{0}   \\
							& 40-48    & 1             & \textbf{0}    & 2            & 1            \\
							& 50-58    & 16            & \textbf{12}   & 15           & 13           \\
							& 60-68    & 27            & \textbf{25}   & \textbf{25}  & \textbf{25}  \\
							& 70-78    & 44            & 38            & 35           & \textbf{31}  \\
							& 80-88    & 61            & 50            & 53           & \textbf{45}  \\
							& 90-100   & 85            & 78            & 80           & \textbf{77}  \\ \hline
		\multirow{8}{*}{8}  & 20-28    & \textbf{0}    & \textbf{0}    & \textbf{0}   & \textbf{0}   \\
							& 30-38    & \textbf{0}    & \textbf{0}    & \textbf{0}   & \textbf{0}   \\
							& 40-48    & \textbf{0}    & \textbf{0}    & \textbf{0}   & \textbf{0}   \\
							& 50-58    & 2             & \textbf{0}    & 2            & \textbf{0}   \\
							& 60-68    & 11            & \textbf{5}    & 10           & \textbf{5}   \\
							& 70-78    & 14            & 10            & 15           & \textbf{8}   \\
							& 80-88    & 36            & 29            & 30           & \textbf{25}  \\
							& 90-100   & 38            & 34            & 36           & \textbf{32}  \\ \hline
		\multirow{8}{*}{11} & 20-28    & \textbf{0}    & \textbf{0}    & \textbf{0}   & \textbf{0}   \\
							& 30-38    & \textbf{0}    & \textbf{0}    & \textbf{0}   & \textbf{0}   \\
							& 40-48    & \textbf{0}    & \textbf{0}    & \textbf{0}   & \textbf{0}   \\
							& 50-58    & 2             & \textbf{0}    & 2            & \textbf{0}   \\
							& 60-68    & 4             & \textbf{2}    & 10           & 5            \\
							& 70-78    & 18            & 9             & 15           & \textbf{8}   \\
							& 80-88    & 20            & \textbf{11}   & 30           & 25           \\
							& 90-100   & 45            & 40            & 36           & \textbf{32}  \botline
		\end{tabular}}
		\label{tab:clustered_timeouts}
		\end{table}

\subsection{Euclidean Generalized TSP}
The instances for the GTSP in the literature~\citep{LAPORTE1981, LAPORTE1987, DBLP:journals/ior/FischettiGT97} 
are normally created from existing structured (see Section~\ref{subsec:structured}) or random (see Section~\ref{subsec:random}) instances.
Starting from random \ac{TSP} instances, we experimented with three approaches to partition points into clusters.

Once a number of clusters $c$ is fixed, the {\em uniform} approach (inspired by \citep{LAPORTE1981, LAPORTE1987}) distributes with uniform probability the points among the clusters.

The \emph{clustering} approach was proposed by~\citet{DBLP:journals/ior/FischettiGT97}. 
For an instance with $n$ points, the number of clusters is set to $c = \lceil n/5\rceil$. Then $c$ points, that will represent the centres of the clusters, are selected to be as far away as possible from each other. The remaining points are then associated with the nearest center.

The  \emph{grid} approach was also proposed by~\citet{DBLP:journals/ior/FischettiGT97}.
Consider the minimum rectangle with sides parallel to the coordinate axes enclosing all points in the instance.
The rectangle is subdivided in an $L \times L$ grid, and each grid cell containing at least one point is a cluster. 
The value $L$ is determined so that each cluster contains an average number $\mu$ of nodes, 
let $\mbox{CLUSTER}(H)$ be the number of clusters corresponding to an $H \times H$ grid, $L$ is defined as the smallest integer such that $\mbox{CLUSTER}(L) \geq n/\mu$.

As in the case of the \ac{ETSP}, we compare two constraint models based on the successor representation. 
The first model, denoted by \reference, includes the basic model for the \ac{EGTSP} described in Equation~\ref{eq:set-constraintmodel}. The second model, called \noi, adds the geometric pruning based on the convex hull introduced in Section~\ref{sec:egtsp_hull}. For all experiments we used the {\em max-regret}~\cite{CaseauLaburthe} search strategy.

We randomly generated instances from 20 to 50 nodes, using the three previously introduced approaches. For instances generated using the grid approach we set the parameter $\mu = 5$.
For each size and each approach we generated 16 instances, for a total of 1488 instances.
Of these 1488 instances, 424 (28.5\%) were excluded from the plots
and aggregate statistics because all compared configurations reached
the 1800-second time limit.

Figure~\ref{graph:egtsp_cactus_plot} compares the experimental results of the two models using a cactus plot, showing the number of solved instances as a function of the CPU time. The \noi\ model consistently outperforms the reference \reference\ model by solving more instances within the same time limit.

In addition to solving a higher number of instances, the \noi\ model also achieves a consistent reduction in both computation time and number of explored nodes across all instance types. 
If we consider instances generated with the clustering approach the average solving time with the \noi\ model is reduced by 76\% and the average number of nodes explored is reduced by 77\%. For grid type instances the results are similar as the reduction of the average solving time results 67\% together with a reduction of the average number of explored nodes by 66\%. 

For uniform type instances the average solving time is reduced by 19\% and the average number of search nodes by 17\%. These slightly lower results can be attributed to the distribution of points that makes it difficult to have sets of neighbours of significant size in the instances. Since the filtering algorithm relies on sets of neighbours their absence weakens the pruning. However, the results overall show that our approach is definitely useful in improving solving performance.

\begin{figure}[tb]
\centering
		\includegraphics[width=\textwidth]{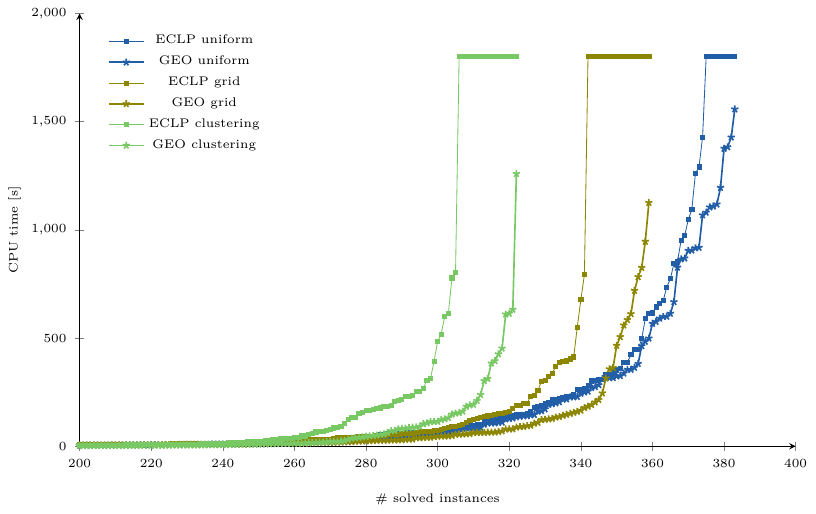}
		\caption{Comparing number of solved instances by \reference\ and \noi\ models on Euclidean \ac{GTSP} instances generated using different approaches to partition points into clusters. The solving time limit was set to 1,800 seconds. We removed instances where both models ran into timeout.}
		\label{graph:egtsp_cactus_plot}
\end{figure}

Overall, the main practical benefit of the proposed geometric
propagators is a reduction of the search space, rather than merely a
lower constant cost per search node. On the random \ac{ETSP}
instances, the number of explored nodes is reduced by 43--75\%,
depending on the instance class and search strategy. This reduction
is accompanied by lower solving times and fewer timeouts, especially
for the larger tested instances. A similar behaviour is observed for
the \ac{EGTSP}, where the geometric model explores fewer nodes and
solves more instances within the time limit. The complexity
improvement discussed for the \nocrossing\ propagator concerns its
implementation relative to a naive propagator; we do not claim an
improvement in the asymptotic complexity of solving the \ac{TSP}.

 \section{Conclusions}
\label{sec:conclusion}

In this article, we showed the effectiveness of the use of geometric information in routing problems.
While the most common approach in Constraint Programming and related research areas is to address the Euclidean Traveling Salesperson Problem with the same constraint model as the general TSP,
we showed that the use of the additional information naturally available in many TSP instances can provide further pruning and speed up the search.
In fact, in the \ac{ETSP} the coordinates of the points in the plane are known, and while the common approach is to use such knowledge only to compute the distance matrix and disregarding it in the constraint model, we employed it to strengthen the constraint model with additional redundant constraints that effectively reduce the search space.

This article is a revised and extended version of a conference
publication~\citet{BertagnonGavanelliAAAI20}, which, to the best of
our knowledge, was the first work to exploit the classical
no-crossing and convex-hull ordering properties through dedicated
propagation mechanisms in a successor-based \ac{CP}/\ac{CLP} model
for the Euclidean TSP.

However, while~\citet{BertagnonGavanelliAAAI20} showed the effectiveness of the idea only in the TSP problem (without any variation that can occur in real life problems),
in this article 
we showed how the use of geometric information can be exploited also in the \ac{EGTSP}, achieving a speedup of up to 4.2$\times$ in solving time and a reduction by a similar factor in the number of explored nodes.
We provided a comprehensive and detailed exposition of the techniques we developed;
a thorough and systematic experimentation campaign showed the effectiveness of the proposed algorithms.

As no CP models for the \ac{EGTSP} were found in the literature, we proposed a first constraint model; 
again, the addition of geometric reasoning let us reduce the average solving time to $1/3$ on clustered instances and to $4/5$ on uniform instances.

Our results are not yet comparable with those obtained by state-of-the-art TSP solvers, such as Concorde, but on the other hand Concorde cannot solve \ac{GTSP} instances. Also note that the implementation of the proposed algorithms is based on a declarative approach and on the open-source \acl{CLP} solver \ECLiPSe. Indeed, the aim of this research was to show that exploiting geometric information can result in stronger pruning rather than proposing a very fast and efficient algorithm which would probably have required an implementation in an imperative language.
We believe that the same techniques can nevertheless also be exploited in solvers implemented in imperative languages.

One limitation of the current work is that it is not universally applicable to all \ac{TSP} variants, such as the \acl{TSPTW}, since its optimal solutions may contain crossings; but crossings avoidance could also be interesting \emph{per se} in some
applications. 
As future work, we plan to investigate extensions of the proposed
techniques to other routing models in which the geometric structure
and the feasible connections preserve the assumptions required by
the proposed pruning.

\section*{Acknowledgements}
Alessandro Bertagnon and Marco Gavanelli are members of the Gruppo Nazionale Calcolo
Scientifico-Istituto Nazionale di Alta Matematica (GNCS-INdAM).

\section*{Competing interests}
The authors declare none.

\bibliographystyle{tlplike}
\bibliography{cited.bib}

\end{document}